\documentclass{article} % For LaTeX2e
\usepackage{paperstyle,times}

\usepackage{amsmath,amsfonts,bm}

\def\eqref#1{equation~\ref{#1}}
\def\plaineqref#1{\ref{#1}}
\def\1{\bm{1}}

\def\eps{{\epsilon}}

\DeclareMathAlphabet{\mathsfit}{\encodingdefault}{\sfdefault}{m}{sl}
\SetMathAlphabet{\mathsfit}{bold}{\encodingdefault}{\sfdefault}{bx}{n}

\newcommand{\E}{\mathbb{E}}

\newif\ifcifarresults
\cifarresultstrue

\newcommand{\CIFARConvClean}{69.20$\pm$0.39}
\newcommand{\CIFARConvASRFive}{82.50$\pm$3.36}
\newcommand{\CIFARConvASRTen}{93.42$\pm$2.07}
\newcommand{\CIFARConvRandFive}{4.32$\pm$0.92}
\newcommand{\CIFARConvRandTen}{8.59$\pm$1.66}
\newcommand{\CIFARSEWClean}{72.02$\pm$0.33}
\newcommand{\CIFARSEWASRFive}{46.23$\pm$1.73}
\newcommand{\CIFARSEWASRTen}{64.98$\pm$3.60}
\newcommand{\CIFARSEWRandFive}{1.70$\pm$0.10}
\newcommand{\CIFARSEWRandTen}{2.11$\pm$0.11}
\newcommand{\CIFARGRUClean}{55.70$\pm$0.35}
\newcommand{\CIFARGRUASRFive}{97.72$\pm$0.47}
\newcommand{\CIFARGRUASRTen}{98.98$\pm$0.22}
\newcommand{\CIFARGRURandFive}{11.67$\pm$0.51}
\newcommand{\CIFARGRURandTen}{22.04$\pm$1.46}
\newcommand{\CIFARTransformerClean}{69.83$\pm$1.26}
\newcommand{\CIFARTransformerASRFive}{78.64$\pm$0.28}
\newcommand{\CIFARTransformerASRTen}{86.00$\pm$0.39}
\newcommand{\CIFARTransformerRandFive}{3.71$\pm$0.65}
\newcommand{\CIFARTransformerRandTen}{5.10$\pm$1.19}

\usepackage{hyperref}
\hypersetup{hidelinks}
\usepackage{url}
\usepackage{amsmath,amssymb,amsthm}
\usepackage{booktabs}
\usepackage{placeins}
\usepackage{float}
\usepackage{multirow}
\usepackage{graphicx}
\usepackage{algorithm}
\usepackage{algpseudocode}
\newtheorem{theorem}{Theorem}
\newtheorem{proposition}{Proposition}
\newtheorem{corollary}{Corollary}

\newcommand{\Dt}{\Delta}
\newcommand{\Null}{\operatorname{Null}}

\title{Stealth Is a Relation, Not a Property:\\
How Event Representations Create Blind Spots for Timing Attacks in Event-Based Perception}

\author{
\hspace*{-\tabcolsep}Shoaib Ahmed Dipu$^{1}$, Md.~Shaown Miah$^{2}$, Kamrul Hasan$^{3}$, Sayeed Shafayet Chowdhury$^{1}$ \\[4pt]
\hspace*{-\tabcolsep}{\normalfont\normalsize $^{1}$Indiana University Indianapolis \quad $^{2}$Bangladesh University of Engineering and Technology} \\
\hspace*{-\tabcolsep}{\normalfont\normalsize $^{3}$Tennessee State University} \\[2pt]
\hspace*{-\tabcolsep}{\normalfont\small\texttt{\{shdipu,\,saychow\}@iu.edu} \quad \texttt{1918018@bme.buet.ac.bd} \quad \texttt{mhasan1@tnstate.edu}}
}

\finalcopy % show the author block

\newcommand{\blfootnote}[1]{%
  \begingroup
  \renewcommand{\thefootnote}{}\footnote{#1}%
  \addtocounter{footnote}{-1}%
  \endgroup
}

\begin{document}

\maketitle
% No venue banner or header rule in the running head.
\lhead{}\chead{}\rhead{}
\renewcommand{\headrulewidth}{0pt}

\blfootnote{Source code is available at \url{https://github.com/shoaibdipu/Stealth_Is_a_Relation_Not_a_Property}.}
\begin{abstract}
An event camera produces an asynchronous stream, but what is visible in that stream depends on how a downstream consumer, such as a model or detector, processes time. The same timestamp change may leave a coarse temporal representation unchanged while changing the response of a model that preserves finer timing. This creates blind spots that timing attacks can exploit. We characterize this dependence as observer-relative stealth. For recorded event streams, retiming an event within its protected accumulation window leaves the accumulated integer tensor exactly unchanged. Any consumer that receives only this tensor, therefore, sees the same input before and after the attack. We use this exact blind space to construct Null, a gradient-guided timestamp-retiming attack, and define $\mathrm{SC\text{-}ASR}_A(\tau)$ to measure attack success while bounding the change visible to observer $A$. Across five event-vision datasets, the exact blind space contains retimings that remain exactly hidden from the protected observer while causing victim failure. On DVS Gesture at a 10\% event budget, Null reaches $81.56\pm5.81\%$ ASR on ConvSNN and $98.67\pm0.45\%$ on a GRU while preserving the protected tensor exactly. On DailyDVS-200, a protocol-scale Multi-View Fusion Network variant reaches $99.28\pm0.11\%$ exact-null ASR, compared with $9.70\pm1.06\%$ for its matched control. In a five-attack comparison, Null is the only method with nonzero attack success at exact observer equality, reaching $81.4\%$ on DVS Gesture and $87.35\%$ on DailyDVS-200. We also search the same exact blind space with an independently implemented constrained projected-gradient optimizer, C-PGD. At matched victim-gradient evaluations, C-PGD reaches $84.50\pm2.89\%$ ASR on DVS Gesture and $89.55\pm4.39\%$ on DailyDVS-200, again with exact protected equality. Both optimization procedures find retimings with high attack success within the exact same blind space. We then change the temporal observer and test the same attacks. Perturbations that are exactly hidden from the protected observer become visible under shifted, finer, overlapping, and randomized temporal views. Adding observer constraints reduces the real-valued blind-space fraction from $87.5\%$ to $75.0\%$ to $62.5\%$, while DVS ConvSNN ASR falls from $74.9\%$ to $61.9\%$ to $37.2\%$. These results show that stealth is not a property of the perturbation alone. It depends on the temporal information available to the observer.
\end{abstract}

\section{Introduction}
\label{sec:intro}

Event cameras produce asynchronous streams in which precise timing carries information~\citep{gallego2020survey,gehrig2019est,peng2023get,sabater2023evtplus,zubic2024ssm}. The same stream can, however, be processed at different temporal resolutions. A temporal model may preserve fine event timing, while another component first accumulates events into coarser frames. The two consumers can therefore receive different information from the same recorded events.

We study the security consequence of this mismatch. A timestamp perturbation can change a fine-time victim while leaving a declared coarse temporal observer exactly unchanged. Prior work establishes timestamps as an attack surface~\citep{yu2026retiming,lee2022aead,marchisio2021dvsattacks,buchel2022eventadv,bu2023rga,yao2024raw,lin2025physical}. Count-preserving temporal transformations can also disappear after temporal aggregation in a poisoning setting~\citep{temporalpoison2026}. We make the observer explicit and study whether its blind space contains test-time retimings that cause victim errors under finite event budgets.

Let $A$ denote the temporal representation available to a downstream consumer. Exact stealth to $A$ means
\begin{equation}
A(X')=A(X).
\label{eq:intro-equality}
\end{equation}
Any computation restricted to $A(X)$ receives exactly the same input before and after the attack. A different observer may still receive different information from the same pair of streams. Stealth therefore depends on the temporal information available to the observer.

Coarse temporal accumulation gives a direct discrete construction of this blind space. Moving an event to another fine bin within the same protected window leaves the accumulated integer count exactly unchanged. We search this feasible set with \emph{Null}, a gradient-guided timestamp-retiming attack that keeps spatial location and polarity fixed. Across N-MNIST, DVS128 Gesture, CIFAR10-DVS, N-Caltech101, and DailyDVS-200, the resulting exact blind space contains retimings that cause victim errors. We observe the same behavior on Spiking Heidelberg Digits (SHD), an auditory event-stream benchmark.

DVS Gesture illustrates the effect. At a 10\% event budget, exact-null attack success rate (ASR) is $81.56\pm5.81\%$ for a convolutional spiking neural network (ConvSNN), $34.06\pm3.69\%$ for a spike-element-wise residual network (SEW-ResNet)~\citep{fang2021sew}, $50.96\pm3.27\%$ for our temporal Transformer~\citep{vaswani2017attention}, and $98.67\pm0.45\%$ for a temporal gated recurrent unit (GRU)~\citep{cho2014gru}. The protected accumulated tensor remains exactly unchanged in every constrained row. When each moved event is restricted to one fine bin, ASR remains $34.7/17.3/58.0/98.7\%$ across the same victims. We report timestamp displacement separately because exact representation equality and physical-time displacement describe different properties of the perturbation.

We next test whether this vulnerability depends on Null's greedy search. Constrained projected gradient descent (C-PGD) provides an independent optimizer over the same unrestricted within-window exact-null feasible set. At matched victim-gradient evaluations on DVS Gesture, C-PGD reaches $84.50\pm2.89\%$ ASR and Null reaches $81.56\pm5.81\%$. Both satisfy $D_A=0$, and additional C-PGD steps increase ASR further. A supporting DailyDVS-200 comparison shows the same qualitative result, with the low clean accuracy of that ConvSNN reported explicitly in Section~\ref{sec:exactnullfair}. Both search procedures find retimings that cause victim errors inside the same exact blind space.

Raw ASR alone does not describe what the observer can see. We define the stealth-constrained attack success rate (SC-ASR)
\begin{equation}
\mathrm{SC\text{-}ASR}_A(\tau)=
\Pr[\text{victim failure}\wedge D_A(X,X')\le\tau\mid\text{clean-correct}],
\label{eq:intro-scasr}
\end{equation}
where $\tau=0$ requires exact protected equality. We report the detector-specific undetected attack success rate (UASR) separately. These quantities distinguish victim failure, representation visibility, and detector evasion. In the five-attack DVS study, Yu and Free achieve high UASR despite nonzero deviation in the protected representation. Null satisfies $D_A=0$ exactly while retaining substantial detector-specific UASR. This separates exact observer-level invisibility from empirical evasion of a particular detector.

An attack hidden from one temporal observer need not remain hidden when the temporal view changes. Canonical-null attacks have zero canonical deviation but become visible under shifted, half-width, overlapping, and randomized temporal views on DVS Gesture and DailyDVS-200. Requiring equality to additional fixed observers reduces the real-valued blind fraction from $87.5\%$ to $75.0\%$ to $62.5\%$. DVS ConvSNN ASR falls from $74.9\%$ to $61.9\%$ to $37.2\%$. Temporal information discarded by one observation can reappear under another.
\vskip 0.25mm
\textbf{Main contributions.}
\textbf{(i) Observer-relative stealth.} We formulate test-time stealth relative to a declared temporal observer and separate raw ASR, $\mathrm{SC\text{-}ASR}_A(\tau)$, and detector-specific UASR. \textbf{(ii) Exact temporal blind space.} We construct exact integer timestamp retimings for coarse temporal accumulation and show that this blind space contains retimings that cause victim errors across datasets and victim families. \textbf{(iii) Independent search validation.} Null and C-PGD independently find retimings with high attack success within the same exact-null feasible set. \textbf{(iv) Observer-dependent visibility.} Changing and intersecting temporal observers changes which perturbations remain hidden, linking observer choice to attack visibility and the size of the common blind space.

\begin{figure}[t]
\centering
\includegraphics[width=\linewidth]{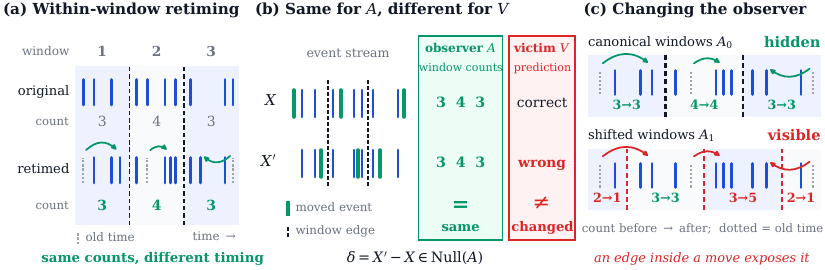}
\caption{\textbf{Observer-relative stealth.} (a) Within-window retiming
preserves window counts. (b) The canonical observer $A$ sees no change,
while a fine-time victim $V$ can. (c) Shifted windows expose the same
perturbation that the canonical windows $A_0 = A$ hide.}
\label{fig:story}
\end{figure}
\section{Related work and novelty boundary}
\label{sec:related-main}

\noindent\textbf{Event-stream attacks.}
Event-stream attacks modify event values, spatial structure, or timestamps~\citep{marchisio2021dvsattacks,lee2022aead,buchel2022eventadv,bu2023rga,yao2024raw,lin2025physical}. They have been studied with models ranging from asynchronous processing to token and transformer representations~\citep{sekikawa2019eventnet,sabater2022evt,jiang2022eventtransformer,peng2023get,sabater2023evtplus}. Sparse Dynamic Attack (SDA), using potential-dependent surrogate gradients (PDSG), targets sparse perturbations of dynamic spiking neural network (SNN) inputs~\citep{lun2025pdsg,neftci2019surrogate,zhou2023spikformer}. Yu et al.~\citep{yu2026retiming} optimize spike retiming under tamper, displacement, and capacity constraints. Temporal Poisoning~\citep{temporalpoison2026} studies a clean-label training-time backdoor whose timestamp transformation preserves per-pixel and per-polarity counts and disappears after temporal aggregation. Our setting is test-time observer-constrained evasion. We define the observer $A$, measure exact and tolerance-based stealth, compare independent optimizers within the same exact-null feasible set, and evaluate the same attacks under different temporal observers.

\noindent\textbf{Stealth and temporal representations.}
Adversarial ML distinguishes attack success from perceptual and detector-level stealth~\citep{frederickson2018dilemma,franzmeyer2024illusory}. Our $\mathrm{SC\text{-}ASR}_A(\tau)$ measures attack success subject to what observer $A$ sees, while $\mathrm{UASR}_\alpha$ measures evasion of a particular detector. These measures capture different properties. An attack can evade a detector while changing the protected representation. Exact $A$-stealth instead guarantees that the declared observer receives the same representation before and after the attack. This distinction matters because event representations retain different amounts of temporal information, from reconstructed intensity sequences and sparse asynchronous processing to task-specific encodings~\citep{rebecq2019hdr,tulyakov2021timelens,zhang2020darkevents,cadena2021spade}. Driving and optical-flow pipelines can retain fine timing~\citep{gehrig2021dsec,gehrig2021eraft}, while converted neuromorphic datasets expose timestamped events directly~\citep{orchard2015nmnist}. We therefore define stealth relative to a named observer instead of a single generic event representation.

\noindent\textbf{Representation invariance.}
Null spaces, Fourier zeros, and Poisson count laws describe information discarded by a temporal observer. For recorded streams, exact integer equality defines the feasible set used by Null and C-PGD and makes observer-level stealth directly measurable. Changing, refining, randomizing, or intersecting temporal observers can expose information hidden from the original observer. Appendix~\ref{app:related} gives the expanded comparison.

\section{Observer-relative stealth and blind spaces}
\label{sec:theory}
Let $R$ be the representation available to an observer. We call $X'$ \emph{exactly stealthy to $R$ at $X$} when $R(X')=R(X)$. For a linear observer $A$, perturbations in $\Null(A)$ form its blind space. Such perturbations can change the underlying event stream without changing the representation seen by the observer.

\begin{proposition}[Observer-restricted detection]\label{prop:detection-impossibility}
If $R(X')=R(X)$, any detector restricted to $R$ receives the same input for the paired clean and attacked samples. A deterministic detector gives the same output. With independent detector randomness, the paired information alone cannot raise the true-positive rate above the false-positive rate.
\end{proposition}
\noindent\textit{Proof.} The claim follows directly from equality of the detector inputs. If $R(X')\neq R(X)$, attack-dependent information is present, although a particular detector may still fail to use it. The formal statement and proof are in Appendix~\ref{app:proofs}.

\paragraph{Exact invariance for recorded events.}
Let $X\in\mathbb N^{T\times C\times H\times W}$ be a fine-bin event tensor. Assume $T=KS$, giving $K$ coarse windows, and let $A$ sum blocks of $S$ consecutive bins. Moving one event from bin $s$ to $s'$ within the same coarse window, while preserving location and polarity, gives $\delta=e_{s'}-e_s$. For several legal moves, write $\Delta=\sum_j\delta_j$.

\begin{proposition}[Exact discrete blind space]\label{prop:discrete-null}
Any finite collection of legal within-window retimings satisfies $A\Delta=0$, so $A(X+\Delta)=AX$ exactly in integer arithmetic. With all aggregation rows present, $\dim\Null(A)=KCHW(S-1)$.
\end{proposition}
\noindent\textit{Proof.} Each move removes and adds one count within the same aggregation row, so its row sum is zero. Linearity gives $A\Delta=0$ for any collection of legal moves. Each length-$S$ aggregation row contributes one constraint and $S-1$ blind degrees of freedom, giving the stated nullity. This exact integer equality defines the feasible set used by Null and C-PGD. The complete proof is in Appendix~\ref{app:proof-discrete-null}.

\paragraph{Changing the observer.}
For linear observers $A_1,\ldots,A_M$,
\begin{proposition}[Intersection of blind spaces]\label{prop:intersection}
\begin{equation}
\bigcap_i\Null(A_i)=\Null([A_1^\top\;\cdots\;A_M^\top]^\top).
\label{eq:joint-null}
\end{equation}
\end{proposition}
\noindent\textit{Proof.} A perturbation lies in the null space of every $A_i$ exactly when their stacked operator annihilates it. Adding observer constraints can only preserve or shrink the common blind space. A retiming hidden from one observer can become visible when another observer retains different temporal information. Section~\ref{sec:defense} tests this prediction with shifted, finer, overlapping, randomized, and intersected observers. The full argument is given in Appendix~\ref{app:proofs}.

\paragraph{Continuous and stochastic views.}
The same observer-relative structure also appears beyond the discrete construction. For a linear temporal kernel, periodic rate perturbations vanish at zeros of its frequency response. Under an inhomogeneous Poisson model, preserving the integrated rate in each disjoint window preserves the joint frame-count distribution. Randomizing accumulation width removes fixed boxcar spectral nulls under the stated conditions. These results require additional assumptions and are separate from the exact recorded-event guarantee used by Null and C-PGD. Their statements, assumptions, and proofs are given in Appendix~\ref{app:proof-kernel} onward.

\section{Null-space attack and three-level evaluation}
\label{sec:attack}

\subsection{Attack construction}
\paragraph{Legal retiming inside the blind space.}
Let an event occupy fine bin $s$ in group $q=(k,c,y,x)$, where $k$ indexes the protected coarse window. A legal destination $s'$ remains in the same $k$ and keeps spatial location and polarity fixed. Let $n_{\mathrm{ev}}$ denote the number of original events. With event budget $b$, at most $\lfloor b n_{\mathrm{ev}}\rfloor$ original events are moved, and each original event is used at most once. For victim loss $\mathcal L$ and $G=\nabla_X\mathcal L$, the first-order gain of a legal move is
\[
\Gamma(q,s\!\to\!s')=G_{q,s'}-G_{q,s}.
\]
We greedily choose high-gain legal moves and periodically refresh the gradient. Every move contributes $e_{s'}-e_s$ within one aggregation row, so the feasible set guarantees $A(X')=A(X)$ exactly. The optimization searches inside the protected blind space instead of balancing attack loss against a similarity penalty.

\paragraph{Independent exact-null optimizer.}
To separate the effect of the feasible set from Null's greedy search, we also use constrained projected gradient descent (C-PGD)~\citep{madry2018pgd}. C-PGD maintains a continuous grouped event tensor. It removes the row-common gradient component, takes normalized projected-gradient steps, and projects each aggregation row back to its fixed event-count simplex. An integer transport projection then converts the relaxed solution back to a timestamp-only event stream under the requested moved-event budget. The resulting stream preserves location and polarity, uses the same unrestricted within-window legal moves as Null, and satisfies $A(X')=A(X)$ exactly. With one C-PGD step per progressive budget stage, the number of victim-gradient evaluations through 10\% matches Null. Both use three evaluations on DVS Gesture and four on DailyDVS-200. Larger step counts are reported only as optimization-strength sensitivity. Appendix~\ref{app:cpgd-spec} gives pseudocode and the integer-projection details.

\paragraph{Controls and comparison attacks.}
We use budget-matched random retiming and, where feasible, an exact signed-displacement-matched control. The latter reproduces each optimized sample's signed shift multiset on randomly chosen legal events. This separates the choice of events from the size and direction of their shifts. The DVS comparison includes free gradient retiming, the official Yu PIL-$L_0$ implementation~\citep{yu2026retiming}, and SDA/PDSG~\citep{lun2025pdsg} and Yao/Gumbel attacks. Yu, Free, and Null share a common nominal retiming operating point. The native feasible sets and objectives of Yu, SDA, Yao, and Free differ from those of the exact-null attacks. We use these methods to study how fixed existing attacks behave under observer-relative evaluation, not as a same-feasible-set optimizer comparison. The direct optimizer comparison is Null versus C-PGD under the same attack population, legal moves, budget path, and matched gradient-evaluation counts. Appendix~\ref{app:fairness} gives the comparison details.

\subsection{Three-level evaluation}
Raw ASR asks only whether the temporal victim changes from correct to incorrect. For protected observer $A$, we first measure representation visibility
\begin{equation}
D_A(X,X')=\frac{\|A(X')-A(X)\|_1}{\|A(X)\|_1},\qquad D_\infty=\|A(X')-A(X)\|_\infty.
\label{eq:repdist}
\end{equation}
We then define stealth-constrained attack success
\begin{equation}
\mathrm{SC\text{-}ASR}_A(\tau)=\Pr[\text{success}\wedge D_A\le\tau\mid\text{clean-correct}],
\label{eq:scasr}
\end{equation}
so $\tau=0$ requires exact equality to the declared observer. For a fixed generated attack set, varying $\tau$ measures how much attack success remains as the observer tolerance is relaxed. The resulting curve is a post-hoc observer-tolerance curve. It is an attack-optimal Pareto frontier only if the attack is re-optimized separately at each $\tau$. Where timestamp displacement is directly comparable, we also report mean absolute displacement $d_t$ and $A(\epsilon)=\Pr[\text{success}\wedge d_t\le\epsilon\mid\text{clean-correct}]$. The two axes answer different questions. The $\tau$ axis limits what the declared observer sees. The $\epsilon$ axis limits how far events move in physical time. We use $A(\epsilon)$ as a complementary temporal-locality diagnostic, not as a substitute for representation stealth.

A practical defender introduces a third quantity. At clean false-positive rate $\alpha$, we define the undetected attack success rate (UASR)
\begin{equation}
\mathrm{UASR}_\alpha=\Pr[\text{success}\wedge\text{defender does not fire}\mid\text{clean-correct}].
\label{eq:uasr}
\end{equation}
The detector is calibrated on held-out clean clips and frozen before attack scoring. Raw ASR measures victim failure. $\mathrm{SC\text{-}ASR}_A$ measures victim failure that also satisfies a declared observer constraint. UASR measures successful attacks that evade a particular detector. We evaluate these quantities separately because a change in representation does not determine whether a particular detector can exploit that change.

\paragraph{Changing the observer.}
Exact stealth is relative to the declared observer. We apply the same generated attacks to shifted, finer-width, overlapping, and randomized temporal observers and measure the resulting representation deviation separately from detector performance. We also impose equality to intersections of fixed observers and measure how the common blind space and attack success change as observer constraints are added. These experiments directly test the observer-relative prediction that information hidden from one temporal view can reappear under another. They measure recovered information and blind-space reduction without assuming that every detector can exploit the recovered signal. Randomized-observer detection is evaluated separately with schedule-aware and schedule-blind rules. Full observer definitions and defense diagnostics are given in Appendix~\ref{app:observer-family}.

\subsection{Evaluation protocol}
Main Null results use three training seeds and several victim families. The unified five-attack studies evaluate Yu PIL-$L_0$, SDA, Yao/Gumbel, Free, and Null over three seeds using common attack populations within each dataset. ASR and SC-ASR are conditioned on clean-correct inputs. The Null--C-PGD comparison uses the same attack population, legal moves, progressive budget path, and matched victim-gradient evaluations at the primary comparison point. Across seeds we report mean$\pm$std descriptively. For repeated-clip comparisons, uncertainty uses a clip-clustered, seed-aware bootstrap that preserves attack pairing and each seed's clean-correct conditioning. Wilson intervals are retained only in per-population tables and are not treated as independent and identically distributed (IID) across repeated seed evaluations. Detector thresholds are calibrated on held-out clean clips and fixed before attack scoring. Appendix~\ref{app:expdetails} gives preprocessing, hyperparameters, clean-pool construction, and statistical details.

\section{Experiments}
\label{sec:experiments}
We first test whether the exact blind space contains retimings that cause victim errors across datasets and victim families. We then repeat the search with constrained projected gradient descent (C-PGD), evaluate existing attacks with raw ASR and $\mathrm{SC\text{-}ASR}_A(\tau)$, and separate representation visibility from detector-specific UASR. Finally, we change the temporal observer and test whether the same perturbations remain hidden. Unless noted, ASR is conditioned on clean-correct samples and headline values are mean$\pm$std over three seeds. Appendix~\ref{app:expdetails} gives preprocessing, model definitions, budgets, per-seed tables, and confidence intervals.

\subsection{Broad exact-null vulnerability}
\label{sec:breadth}
The protected operator is coarse temporal accumulation, and Null moves events only within their original protected windows. Every optimized row therefore satisfies $A(X')=A(X)$ exactly in integer arithmetic. Table~\ref{tab:broad} shows the common 10\% operating point for four aligned event-vision datasets and four victim families. The ConvSNN is our compact convolutional SNN; SEW-ResNet follows spike-element-wise residual learning~\citep{fang2021sew}; the temporal GRU follows gated recurrent units~\citep{cho2014gru}; and our Event Transformer uses a standard Transformer encoder~\citep{vaswani2017attention}. Null exceeds the corresponding random control in every dataset--victim pair, although sensitivity inside the blind space varies substantially by architecture.

\begin{table}[!htbp]
\caption{\textbf{Exact-null attack across four aligned event-vision datasets.} Optimized ASR / random-control ASR (\%). N-MNIST is reported separately below because its original and aligned protocols use different victim sets. DVS ConvSNN/SEW use displacement-conditioned random controls. DVS Transformer/GRU and the later dataset protocols use exact signed-displacement matching where reported. All optimized rows satisfy \(D_A=D_\infty=0\).}
\label{tab:broad}
\centering\tiny
\renewcommand{\arraystretch}{0.92}
\setlength{\tabcolsep}{2.8pt}
\begin{tabular}{lrrrr}
\toprule
Dataset & ConvSNN & SEW & Transformer & GRU \\
\midrule
DVS Gesture & 81.6 / 0.6 & 34.1 / 0.8 & 51.0 / 1.1 & 98.7 / 2.2 \\
CIFAR10-DVS & 93.4 / 8.6 & 65.0 / 2.1 & 86.0 / 5.1 & 99.0 / 22.0 \\
N-Caltech101 & 44.5 / 1.7 & 36.2 / 1.3 & 70.7 / 2.3 & 89.5 / 8.8 \\
DailyDVS-200 & 84.3 / 9.8 & 34.2 / 9.3 & 98.3 / 13.6 & 98.9 / 16.2 \\
\bottomrule
\end{tabular}
\end{table}

On DVS Gesture, the protected CoarseFrameFormer reaches $93.81\pm0.22\%$ clean accuracy and its logits remain bitwise identical on exact-null rows (Appendix~\ref{app:strongtransformer}). N-MNIST provides an additional aligned protocol check, DailyDVS-200 adds protocol-scale ACTION-Net~\citep{wang2021action}, MVFNet~\citep{wu2021mvfnet}, Swin-T~\citep{liu2021swin}, and TimeSformer~\citep{bertasius2021timesformer} variants, and SHD tests the same discrete mechanism outside vision. These supporting results are reported in Appendices~\ref{app:nmnist-aligned}--\ref{app:shd}.

\subsection{Null versus C-PGD in the same feasible set}
\label{sec:exactnullfair}
Figure~\ref{fig:cpgd-main} compares Null with C-PGD inside the same exact-null feasible set. Null and C-PGD use the same ConvSNN victim, attack population, timestamp-only within-window moves, and requested 10\% moved-event budget. At the primary comparison point, C-PGD uses one projected-gradient step per progressive budget stage, matching Null's victim-gradient evaluations through 10\%.

\begin{figure}[H]
\centering
\includegraphics[width=.46\linewidth]{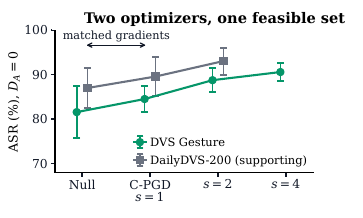}
\caption{\textbf{Two optimizers, one exact-null feasible set.} ASR at $D_A=0$ (mean$\pm$std, three seeds). C-PGD $s=1$ matches Null's victim-gradient evaluations; $s=2,4$ test optimization sensitivity. DailyDVS-200 is supporting because its ConvSNN has low clean accuracy.}
\label{fig:cpgd-main}
\end{figure}

At matched victim-gradient evaluations, C-PGD reaches $84.50\pm2.89\%$ versus $81.56\pm5.81\%$ for Null on DVS, a paired difference of $+2.94$ points with 95\% CI $[0.85,5.12]$. DailyDVS gives $89.55\pm4.39\%$ versus $86.98\pm4.44\%$, with paired CI $[-0.29,5.56]$. Additional C-PGD steps increase ASR while preserving $D_A=0$. The DVS comparison is primary; the DailyDVS ConvSNN result is supporting because its clean accuracy is $13.55\pm0.60\%$. Appendix~\ref{app:fairness} gives the full comparison, including the progressive-versus-single-pass DailyDVS reconciliation.

\subsection{Existing attacks under observer-relative evaluation}
\label{sec:directcomparison}
The five-attack study keeps each method's native objective and scores its generated attacks with the same observer-relative metric. We compare Yu PIL-$L_0$~\citep{yu2026retiming}, Sparse Dynamic Attack (SDA) with potential-dependent surrogate gradients~\citep{lun2025pdsg}, Yao/Gumbel~\citep{yao2024raw}, unconstrained Free retiming, and Null. Because their feasible sets differ, $\tau$ is applied \emph{post hoc}. Figure~\ref{fig:tau-two} is therefore a descriptive $\mathrm{SC\text{-}ASR}(\tau)$ comparison, not an attack-optimal Pareto frontier.

\begin{figure}[H]
\centering
\includegraphics[width=0.80\linewidth]{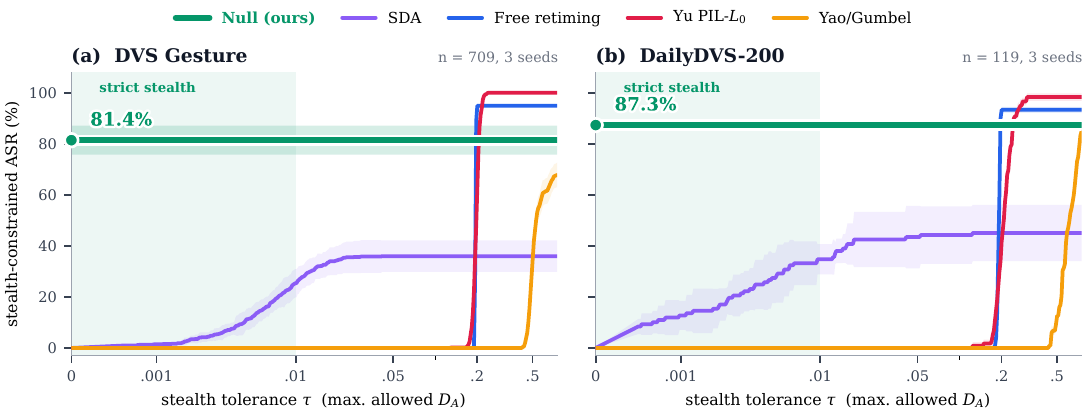}
\caption{\textbf{Existing attacks under a post-hoc observer tolerance.} DVS Gesture (left) and DailyDVS-200 (right). $\tau$ filters fixed generated attacks by protected-representation deviation. Null is generated inside the exact-null feasible set and is the only method here with nonzero success at $\tau=0$.}
\label{fig:tau-two}
\end{figure}

At $\tau=0$, Null retains 81.4\% attack success on DVS and 87.35\% on DailyDVS, while the other four generated attack sets have zero success. Raw ASR is higher for several alternatives, showing that victim failure and observer-constrained success are different quantities. Temporal locality is also separate: in the unified seed-0 DVS run, $A(200\,\mathrm{ms})$ is 51.88\% for Null, 7.11\% for Yu, and 0\% for Free. Full per-seed, footprint, temporal-radius, and statistical results are in Appendices~\ref{app:direct-eps}, \ref{app:fairness}, and~\ref{app:operational-partial}.

\subsection{Representation visibility versus operational detection}
\label{sec:operational}
Representation visibility does not determine detector evasion. Figure~\ref{fig:operational-main} therefore reports undetected attack success rate (UASR) separately from raw ASR and exact $\mathrm{SC\text{-}ASR}_A(0)$. The same clean-only detector protocol is applied to all five DVS attacks.

\begin{figure}[H]
\centering
\includegraphics[width=.53\linewidth]{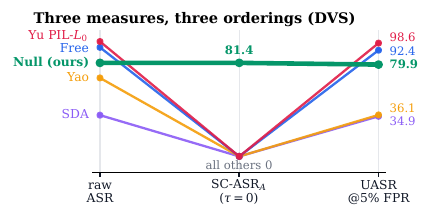}
\caption{\textbf{Three measures, three orderings on DVS Gesture.} Raw ASR measures victim failure, $\mathrm{SC\text{-}ASR}_A(0)$ requires exact observer equality, and UASR is successful attack rate after a detector at 5\% nominal false-positive rate. The same attack can rank very differently under the three questions.}
\label{fig:operational-main}
\end{figure}

Yu and Free retain high UASR despite changing the protected representation, whereas Yao is much easier to detect. Null satisfies $D_A=0$ exactly while retaining 79.9\% UASR. Yao's detector AUROC is 0.929; the remaining attacks are near chance in this DVS experiment. Complete AUROC, calibration, and DailyDVS diagnostic results are in Appendix~\ref{app:operational-partial}.

\subsection{Cross-observer visibility}
\label{sec:defense}
We next keep the generated Null attacks fixed and change only the temporal observer. Exact equality to the canonical observer need not survive a different temporal view. Table~\ref{tab:xobs} shows zero canonical deviation but nonzero deviation under shifted, half-width, and randomized observers on both DVS and DailyDVS. Overlapping and multiscale diagnostics in Appendix~\ref{app:observer-family} show the same qualitative effect.

\begin{table}[t]
\caption{\textbf{Changing the observer changes visibility and the common blind space.} The top reports mean relative deviation of the same Null attacks under alternative observers over three seeds. The bottom reports DVS attacks re-optimized while enforcing equality to additional fixed phase observers.}
\label{tab:xobs}
\centering\tiny
\renewcommand{\arraystretch}{0.92}
\setlength{\tabcolsep}{3.2pt}
\begin{tabular}{lrrrrrr}
\toprule
\multicolumn{7}{l}{\textit{Same attack, different observer. Relative representation deviation}}\\
Dataset & Canonical & Shift-1 & Shift-2 & Shift-4 & Half-width & Random mean\\
\midrule
DVS Gesture & \textbf{0.000} & 0.081 & 0.116 & 0.137 & 0.144 & 0.103\\
DailyDVS-200 & \textbf{0.000} & 0.078 & 0.111 & 0.131 & 0.143 & 0.099\\
\midrule
\multicolumn{7}{l}{\textit{DVS observer intersections. Blind-space size and attack success}}\\
Observer constraints & \multicolumn{2}{c}{Phase 0} & \multicolumn{2}{c}{Phases 0+2} & \multicolumn{2}{c}{Phases 0+2+4}\\
Blind fraction (\%) & \multicolumn{2}{c}{87.5} & \multicolumn{2}{c}{75.0} & \multicolumn{2}{c}{62.5}\\
ConvSNN ASR (\%) & \multicolumn{2}{c}{74.9} & \multicolumn{2}{c}{61.9} & \multicolumn{2}{c}{37.2}\\
SEW ASR (\%) & \multicolumn{2}{c}{38.0} & \multicolumn{2}{c}{29.5} & \multicolumn{2}{c}{20.7}\\
\bottomrule
\end{tabular}
\end{table}

The top panel supports the relational claim that canonical invisibility does not transfer automatically to another observer. The bottom panel tests the corresponding geometric prediction. Adding phase constraints shrinks the real-valued blind space and reduces ConvSNN and SEW attack success. GRU remains near saturation in this experiment (Appendix~\ref{app:observer-family}), so the intersection experiment does not establish a universal defense. Randomized schedules expose the fixed canonical-null DVS attacks in about 97.5\% of realized schedules. This is a representation-exposure result, not a complete defense benchmark. In the controlled randomized-accumulation study, a schedule-aware detector reaches AUROC about \(0.97/0.999\), while a schedule-blind statistic remains near chance (Appendix~\ref{app:controlled}). The temporal observation rule therefore changes both the available evidence and the feasible blind space.

\subsection{Controls and ablations}
\label{sec:ablations}
Targeted controls test simpler explanations. With a one-fine-bin cap, DVS ConvSNN/SEW/Transformer/GRU ASR remains 34.7/17.3/58.0/98.7\%. A matched-effort CIFAR10-DVS comparison shows that only Null preserves the protected representation, and all 12 optimized-versus-control comparisons remain significant after Holm correction. In the controlled study, null-subspace projected gradient descent at $\epsilon=.25$ lowers LIF accuracy to $0.33\%$ while the protected frame CNN remains at $97.80\%$, with no protected-label flips in $21{,}000$ attacks. Further ablations are in Appendices~\ref{app:shift}, \ref{app:observer-family}, and~\ref{app:cost}.

\section{Discussion and conclusion}
\label{sec:limitations}
\par\vspace{-0.35ex}
\noindent\textbf{Damaging directions inside an exact blind space.} Coarse temporal accumulation discards within-window timing, yet changes along these directions can affect models that retain finer timing. Null leaves the protected integer tensor exactly unchanged while causing victim errors. On DVS Gesture at 10\%, exact-null ASR reaches 81.56\% for ConvSNN and 98.67\% for GRU. C-PGD independently finds retimings with similar attack success inside the same exact-null feasible set. The vulnerability follows the blind space, not Null's particular search procedure. For the DVS temporal GRU and Event Transformer, matched controls preserve the optimized signed-shift multiset while assigning it to different legal event identities. All 6,336 sample-level matches complete exactly with zero fallback, and the controls remain far below the optimized attacks.
\par\vspace{-0.35ex}
\noindent\textbf{Attack success, observer stealth, and detection are distinct.} Yu and Free achieve high raw ASR and UASR while changing the protected representation; Null retains attack success at exact equality, while Yao has larger representation deviation and is easier to detect on DVS. Raw ASR, $\mathrm{SC\text{-}ASR}_A(\tau)$, and UASR measure victim failure, observer-constrained success, and detector evasion, respectively. None should be used as a proxy for the others. Protected CoarseFrameFormer logits remain bitwise identical on the DVS exact-null evaluations, while the operational detector can use information outside that representation.
\par\vspace{-0.35ex}
\noindent\textbf{Changing the observer changes what remains hidden.} Canonical-null attacks become visible under shifted, finer, overlapping, and randomized temporal views. Adding fixed observer constraints reduces the real-valued blind fraction from 87.5\% to 75.0\% to 62.5\%, while DVS ConvSNN ASR falls from 74.9\% to 61.9\% to 37.2\%. In the controlled randomized-schedule study, a schedule-aware detector reaches AUROC about 0.97/0.999 while a schedule-blind statistic remains near chance. Visibility therefore depends on the temporal information available to the observer, although these experiments do not establish observer diversification as a complete defense. GRU remains near saturation in the fixed-observer intersection experiment, so this is not a universal defense result.
\par\vspace{-0.35ex}
\noindent\textbf{Temporal locality is a separate constraint.} The tolerance $\tau$ bounds protected-representation change, whereas $A(\epsilon)$ measures success within a timestamp-displacement radius. For the three retiming methods in the seed-0 DVS analysis, $A(200\,\mathrm{ms})$ is 51.88\% for Null, 7.11\% for Yu, and 0\% for Free. Under a stricter one-fine-bin DVS cap, ASR remains 34.7/17.3/58.0/98.7\% for ConvSNN/SEW/Transformer/GRU. Exact observer equality and temporal displacement therefore describe different properties of the perturbation. We therefore report representation deviation, event footprint, and timestamp displacement separately.
\par\vspace{-0.35ex}
\noindent\textbf{Statistical support and scope.} Across three training seeds where applicable, the paired C-PGD$-$Null difference is $+2.94$ points on DVS (95\% CI $[0.85,5.12]$) and $+2.57$ on DailyDVS ($[-0.29,5.56]$). SHD preserves exact protected equality with substantial ASR, extending the discrete mechanism beyond event vision; vulnerability remains architecture dependent. The breadth study retains low-sensitivity victims rather than selecting only architectures on which the attack is strong. Our threat model is digital post-sensor timestamp retiming, not an over-the-air physical attack. The appendix reports the matching, equality, population, and calibration checks.
\par\vspace{-0.35ex}
\noindent\textbf{Conclusion.} Stealth depends jointly on the perturbation and the observer. A retimed stream can be unchanged to one temporal representation, cause errors in a finer-time model, and become visible under another view. $\mathrm{SC\text{-}ASR}_A(\tau)$ measures observer-constrained success, $A(\epsilon)$ measures temporal locality, and UASR measures detector evasion. Observer intersections show how additional temporal information changes the common blind space.

%\clearpage

\subsection*{AI use statement}
We used generative AI tools for research brainstorming and conceptual framing, discussion of mathematical arguments and experimental design, code drafting and debugging, analysis scripts, result interpretation, and language editing. The authors made the final methodological decisions, ran all reported experiments, checked numerical claims against the saved result files, reviewed AI-assisted mathematical arguments, code, and prose, and take responsibility for the final manuscript, claims, citations, and artifacts.

\subsection*{Ethics statement}
This work studies adversarial manipulation of event-camera streams and has dual-use implications. The empirical study is limited to digital timestamp retiming of recorded events. It does not demonstrate an over-the-air physical attack. N-MNIST, DVS128 Gesture, CIFAR10-DVS, N-Caltech101, and DailyDVS-200 are public benchmark datasets, and we collected no new human-subject data. DVS128 Gesture contains recordings of participants performing gestures. We use only the released benchmark events and labels for recognition experiments.

\subsection*{Reproducibility statement}
Appendix~\ref{app:proofs} contains full proofs and assumptions. Appendix~\ref{app:expdetails} documents preprocessing, model definitions, budgets, controls, and ablations. Appendix~\ref{app:controlled} gives the controlled spectral study. The experiment code and the saved per-seed result tables used for the reported values are available at \url{https://github.com/shoaibdipu/Stealth_Is_a_Relation_Not_a_Property}.

% Additional directly relevant architecture/representation references retained in the bibliography for context.
\nocite{amir2017gesture,fang2021sew,cho2014gru,vaswani2017attention,lagorce2017hots,sironi2018hats}
\bibliography{refs}
\bibliographystyle{paperstyle}

\appendix
\newpage
\section{Proofs and theoretical details}
\label{app:proofs}

\subsection{Continuous and stochastic statements}
\label{app:continuous-theory}
This section gives the continuous-time and stochastic results summarized in Section~\ref{sec:theory}.
\begin{theorem}[Kernel-zero criterion]\label{thm:kernelzero}
For the linear response $y_k=\int g(t_k-t)\,dN(t)$ and a nonzero additive periodic rate perturbation $\delta\lambda(t)=\alpha\cos(2\pi ft+\varphi)$ with $\alpha\neq0$, $\Delta\mathbb E[y_k]=0$ for every $k$ and phase $\varphi$ iff $\widehat g(f)=0$.
\end{theorem}
\begin{corollary}[Boxcar nulls]\label{cor:harmonics}
For a boxcar of width $\Delta$, nonzero zeros occur at $f=m/\Delta$. An idealized leaky integrate-and-fire (LIF) kernel has nonzero gain at every finite $f$.
\end{corollary}
\begin{theorem}[Poisson count indistinguishability]\label{thm:poisson}
For an inhomogeneous Poisson process, preserving integrated rate in every disjoint accumulation window preserves the joint distribution of the complete frame-count sequence.
\end{theorem}
\begin{theorem}[Randomized accumulation]\label{thm:randomized}
For positive random boxcar width $W$ with characteristic function $\phi_W$,
\begin{equation}\mathbb E_W|\widehat g_W(f)|^2=\frac{1-\mathrm{Re}\,\phi_W(2\pi f)}{2\pi^2f^2}.\label{eq:random-visibility}\end{equation}
If $W$ has a density, this energy is positive for every $f\neq0$.
\end{theorem}

\subsection{Proof of Proposition~\ref{prop:detection-impossibility}}
\label{app:proof-detection}
Let a detector restricted to $R$ be written as $D(R(X),U)$, where $U$ denotes any detector randomness. If $R(X')=R(X)$, then for every fixed realization $u$,
\[
D(R(X'),u)=D(R(X),u).
\]
A deterministic detector is the special case without $U$ and returns the same output on the paired clean and attacked samples. If the clean and attacked evaluations use independent randomness with the same law, their output distributions are identical because their representation inputs are identical. The paired representation therefore contains no information that can increase true-positive rate above false-positive rate. This statement applies only to detectors that factor through $R$. A detector with access to the raw event stream or another representation is not covered.

\subsection{Baseline-rate scope}
\label{app:baseline-scope}
For a general multiplicative perturbation $\lambda(t)\mapsto\lambda(t)(1+a(t))$, the attack-induced change is
\[
\Delta\E[y_k]=\int g(t_k-t)\lambda(t)a(t)\,dt.
\]
The transform condition applies to the product $\lambda a$. Writing $\lambda(t)=\lambda_0$ over the support of $g$ factors out the baseline and gives the pure-tone form used in Theorem~\ref{thm:kernelzero}. If the baseline changes inside a window, a sinusoidal $a$ need not remain single-frequency after multiplication by $\lambda$. In the original controlled check, a linear-ramp baseline produced a residual about $7\times10^{-3}$ of the frame mean and a full saccade profile produced about $9\times10^{-2}$. The recorded-event experiments bypass this approximation by enforcing exact per-window sums.

\subsection{Proof of Theorem~\ref{thm:kernelzero}}
\label{app:proof-kernel}
With $\delta\lambda(t)=\alpha\cos(2\pi ft+\varphi)$, Campbell's formula gives
\[
\Delta\E[y_k]=\alpha\int g(t_k-t)\cos(2\pi ft+\varphi)\,dt.
\]
Substitute $u=t_k-t$:
\begin{align*}
\Delta\E[y_k]
&=\alpha\int g(u)\cos\!\left(2\pi f(t_k-u)+\varphi\right)\,du\\
&=\alpha\,\operatorname{Re}\!\left\{e^{j(2\pi ft_k+\varphi)}\int g(u)e^{-j2\pi fu}\,du\right\}\\
&=\alpha|\widehat g(f)|\cos\!\left(2\pi ft_k+\varphi+\angle\widehat g(f)\right).
\end{align*}
If $\widehat g(f)=0$, the response vanishes for every $k$ and $\varphi$. Conversely, assume $\widehat g(f)\neq0$. Set $\varphi=-2\pi ft_k-\angle\widehat g(f)$ for any chosen $k$. Then $\Delta\E[y_k]=\alpha|\widehat g(f)|\neq0$. Linearity extends the claim to additive perturbations whose spectral support lies in the zero set.

\subsection{Proof of Corollary~\ref{cor:harmonics}}
\label{app:proof-harmonics}
For $g=\mathbf 1_{[0,\Dt)}$,
\[
\widehat g(f)=\int_0^\Dt e^{-j2\pi fu}\,du
=\Dt\,\mathrm{sinc}(f\Dt)e^{-j\pi f\Dt},
\]
whose nonzero zeros satisfy $f\Dt\in\mathbb Z\setminus\{0\}$. For $h(t)=e^{-t/\tau}\mathbf 1_{t\ge0}/\tau$,
\[
\widehat h(f)=\frac{1}{1+j2\pi f\tau},\qquad
|\widehat h(f)|=\left(1+(2\pi f\tau)^2\right)^{-1/2}>0.
\]
At the first boxcar null $f=1/\Dt$, the LIF gain is $(1+(2\pi\tau/\Dt)^2)^{-1/2}$, which equals about $0.62$ when $\tau/\Dt=0.2$.

\subsection{Proof of Theorem~\ref{thm:poisson}}
\label{app:proof-poisson}
For an inhomogeneous Poisson process with intensity $\lambda$, the count in a Borel set $B$ is Poisson with mean $\int_B\lambda$, and counts in disjoint sets are independent. Let $y_k=N([k\Dt,(k+1)\Dt))$. If the perturbation has zero integral in every coarse window, then each attacked window has the same Poisson parameter as its clean counterpart. Independence holds in both cases, so the product law of the full count sequence is identical. For any test $\psi$ measurable with respect to $(y_k)_k$, $\E_{H_1}\psi=\E_{H_0}\psi$. Its power equals its size.

This conclusion is exact for Poisson window counts and for any statistic computed only from those counts. For a non-Poisson sensor process, equality of the window integrals guarantees equality of first moments but does not by itself guarantee equality of higher-order count statistics. Proposition~\ref{prop:discrete-null} supplies the stronger realized-representation statement used in the recorded-event experiments.

\subsection{Proof of Proposition~\ref{prop:intersection}}
\label{app:proof-intersection}
A perturbation $\delta$ is invisible to every monitored linear representation precisely when $A_i\delta=0$ for all $i$. Stacking these equations gives
\[
\begin{bmatrix}A_1\\ \vdots\\ A_M\end{bmatrix}\delta=0,
\]
which proves Eq.~(\plaineqref{eq:joint-null}). Appending rows to a linear operator cannot decrease its rank and cannot increase the dimension of its null space.

\subsection{Proof of Proposition~\ref{prop:discrete-null}}
\label{app:proof-discrete-null}
For one coarse window and fixed $(c,y,x)$, the aggregation row is $\mathbf 1_S^\top$. Moving one event from $s$ to $s'$ produces $\delta=e_{s'}-e_s$, hence
\[
\mathbf 1_S^\top\delta=1-1=0.
\]
All other coarse rows are unchanged, so $A\delta=0$. For a finite collection of legal moves, define the aggregate perturbation $\Delta=\sum_i\delta_i$. Linearity gives $A\Delta=0$, and therefore $A(X+\Delta)=AX$ exactly in integer arithmetic. Any downstream function $C$ of $AX$ receives the identical argument.

\subsection{Nullity calculation for Proposition~\ref{prop:discrete-null}}
\label{app:proof-nullity}
Each fixed tuple $(k,c,y,x)$ contributes one row that sums $S$ distinct fine-time coordinates. Rows for distinct tuples have disjoint support, so the $KCHW$ rows are linearly independent and $\operatorname{rank}(A)=KCHW$. The input dimension is $KSCHW$. Rank-nullity gives
\[
\dim\Null(A)=KSCHW-KCHW=KCHW(S-1).
\]
This is the dimension over the reals. Legal timestamp moves occupy an integer subset further restricted by event occupancy and nonnegativity.

\subsection{Derivation of the mismatch law, Eq.~(\plaineqref{eq:rho})}
\label{app:mismatch-law}
For fixed $m\in\mathbb Z\setminus\{0\}$ and small $\eps$, let $f=(m+\eps)/\Dt$. Then
\[
\mathrm{sinc}(m+\eps)
=\frac{\sin(\pi(m+\eps))}{\pi(m+\eps)}
=\frac{(-1)^m\sin(\pi\eps)}{\pi(m+\eps)}.
\]
Hence
\[
|\widehat g(f)|^2
=\Dt^2\frac{\sin^2(\pi\eps)}{\pi^2(m+\eps)^2}
=\Dt^2\frac{\eps^2}{m^2}\left(1+O(\eps)\right),
\]
using $\sin(\pi\eps)=\pi\eps(1+O(\eps^2))$ and $(m+\eps)^{-2}=m^{-2}(1+O(\eps))$. Dividing $\Dt^2|\widehat h(f)|^2$ by this quantity gives
\begin{equation}
\rho(f)=\frac{m^2}{\eps^2}\left(1+(2\pi f\tau)^2\right)^{-1}(1+O(\eps)).
\label{eq:rho}
\end{equation}
At $\eps=0$, the denominator is zero and $\rho$ diverges.

\subsection{Proof of Theorem~\ref{thm:randomized}}
\label{app:proof-randomized}
For a boxcar of random width $W$,
\[
\widehat g_W(f)=\int_0^W e^{-j2\pi fu}\,du
=\frac{1-e^{-j2\pi fW}}{j2\pi f}.
\]
Thus
\[
|\widehat g_W(f)|^2
=\frac{|1-e^{-j2\pi fW}|^2}{4\pi^2f^2}
=\frac{1-\cos(2\pi fW)}{2\pi^2f^2}.
\]
Taking the expectation and using $\E[\cos(2\pi fW)]=\operatorname{Re}\phi_W(2\pi f)$ gives Eq.~(\plaineqref{eq:random-visibility}). If this expectation is zero, then $\cos(2\pi fW)=1$ almost surely, so $W$ lies on the countable lattice $\{m/f:m\in\mathbb Z\}$ almost surely. A distribution with a density assigns zero mass to that set. The expectation is strictly positive for $f\neq0$.

For $W\sim\operatorname{Unif}[\Dt(1-r),\Dt(1+r)]$,
\[
\phi_W(\omega)=e^{j\omega\Dt}\,\mathrm{sinc}(\omega\Dt r/\pi),
\]
so at $f=m/\Dt$,
\[
\E|\widehat g_W(m/\Dt)|^2
=\frac{\Dt^2\bigl(1-\mathrm{sinc}(2mr)\bigr)}{2\pi^2m^2}.
\]
For $m=1$ this increases with $r$ on $[0,1/2]$. At $r=0.25$ it equals $1.84\times10^{-2}\Dt^2$. A continuous width law is needed. A non-degenerate law supported on $\{1/f,2/f\}$ still gives zero visibility at frequency $f$.

\subsection{Schedule-aware deflection: full derivation}
\label{app:schedule-deflection}
Let window $k$ have width $W_k$ and center $t_k^c$. Integrating the attacked rate over that window gives
\begin{equation}
\mu_k=\lambda_0W_k+\lambda_0\alpha\,G(W_k)\cos(2\pi f t_k^c+\varphi),
\qquad
G(W)=\frac{\sin(\pi fW)}{\pi f}.
\label{eq:app-mu}
\end{equation}
At $f=m/\Dt$, write $W/\Dt=1+u$. If the law of $W$ is symmetric about $\Dt$, then $u$ is symmetric and
\[
\sin(\pi m(1+u))=(-1)^m\sin(\pi m u),
\]
an odd function of $u$. Hence $\E[G(W)]=0$.

For a schedule-blind coherent statistic, use the nominal schedule rather than the realized width and set $\widetilde z_k=(y_k-\lambda_0\Dt)/\sqrt{\lambda_0\Dt}$. If the nominal phase is independent of $W_k$, the attack-induced mean factorizes and symmetry about $\Dt$ gives $\E[G(W_k)]=0$ at the old null. This exact cancellation need not hold for arbitrary randomized schedules or for a statistic normalized by $1/\sqrt{W_k}$.

Now use the realized gain. Set
\[
z_k=\frac{y_k-\lambda_0W_k}{\sqrt{\lambda_0W_k}},\qquad
w_k=\frac{G(W_k)e^{-j2\pi f t_k^c}}{\sqrt{W_k}},
\qquad
S=\sum_k z_kw_k.
\]
For conditionally independent Poisson counts in disjoint windows, under the null $\operatorname{Var}(S)=\sum_kG(W_k)^2/W_k$. From Eq.~(\plaineqref{eq:app-mu}), the attacked mean of $z_k$ is
\[
\E[z_k]=\sqrt{\lambda_0}\alpha\frac{G(W_k)}{\sqrt{W_k}}
\cos(2\pi f t_k^c+\varphi).
\]
Multiplying by $w_k$ and summing gives
\[
\E[S]
=\sqrt{\lambda_0}\alpha\sum_k\frac{G(W_k)^2}{W_k}
\cos(2\pi f t_k^c+\varphi)e^{-j2\pi f t_k^c}.
\]
Assume the widths are independent, or more generally mixing, and that the accumulated centers $t_k^c$ yield sufficient phase dispersion at $2f$. Then the normalized double-frequency sum vanishes asymptotically. For finite $K$ this is an approximation, leaving
\[
\E[S]\approx \frac{\sqrt{\lambda_0}\alpha}{2}e^{j\varphi}
\sum_k\frac{G(W_k)^2}{W_k}.
\]
The squared deflection is
\begin{equation}
d^2=\frac{|\E[S]|^2}{\operatorname{Var}(S)}
=\frac{\alpha^2\lambda_0}{4}\sum_k\frac{G(W_k)^2}{W_k}
\approx\frac{K\alpha^2\lambda_0\E|\widehat g_W(f)|^2}{4\Dt}.
\label{eq:deflection}
\end{equation}
For widths concentrated around $\Dt$, the approximation replaces $1/W_k$ by $1/\Dt$ and the empirical mean of $G(W_k)^2$ by its expectation. Theorem~\ref{thm:randomized} makes this expectation positive, so $d^2$ grows linearly with the number of observed windows.

\subsection{Additional representation and temporal-budget details}
\label{app:theory-additional}
Exact stealth corresponds to the zero-radius case of the representation-space ball $\mathcal B_\tau^A(X)$. The radius $\tau$ is a monitor tolerance rather than a second event-retiming budget. The same perturbed event stream can therefore have different stealth values for different representations $A$. For the block-sum operator, Fourier carriers occupy structured directions inside the larger aggregation null space. Legal timestamp moves occupy an integer subset further restricted by event occupancy and nonnegativity.

A $B$-bin voxel grid changes the temporal kernel. At the first null of the original frame, normalized visibility becomes $|\mathrm{sinc}(1/B)|$. It is $0$ at $B=1$, $0.64$ at $B=2$, and $0.98$ at $B=10$. Finer temporal summaries close one blind spot and create their own higher-frequency nulls. For the N-MNIST, DVS Gesture, CIFAR10-DVS, N-Caltech101, and DailyDVS-200 tensorizations used in Section~\ref{sec:experiments}, Proposition~\ref{prop:discrete-null} gives $312{,}120$, $1{,}146{,}880$, $573{,}440$, $573{,}440$, and $573{,}440$ real-valued blind directions, respectively.

The representation budget $\tau$ and temporal budget $\epsilon$ constrain different properties. The first limits what the protected observer sees. The second limits how far the attack moves events in time. A perturbation may lie in $\Null(A)$ and hence have $D_A=0$ while still differing substantially in $d_t$. Reporting both $\mathrm{SC\text{-}ASR}_A(\tau)$ and $A(\epsilon)$ separates observer-relative stealth from the amount of retiming required to achieve it.

\section{Supplementary comparisons and figures}
\label{app:suppcompare}
This section collects two supporting views that are useful for interpreting the main result but are not needed for the core experimental sequence: a direct comparison with prior event-security work and a visualization of the continuous transfer-function mechanism. The first clarifies the novelty boundary; the second shows why the same temporal modulation can be hidden from one representation and visible to another.

\FloatBarrier
\subsection{Event-security comparison}
\label{app:related}
\begin{table}[!htbp]
\caption{Comparison with event-security work. ``Protected observation'' denotes a guarantee on the realized representation available to a stated observer. Collapsed counts are one specific observation. ``Exact chosen $R$'' denotes the general invariance condition $R(X')=R(X)$, which reduces to count equality when $R$ is frame accumulation.}
\label{tab:related}
\centering
\scriptsize
\setlength{\tabcolsep}{2.5pt}
\begin{tabular}{p{1.75cm}p{2.45cm}c p{2.25cm} p{2.55cm}}
\toprule
Work & Perturbation & Count-preserving & Protected observation & Visibility / defense analysis \\
\midrule
DVS-Attacks~\citep{marchisio2021dvsattacks} & event-sequence perturbation & no constraint & -- & sensor-filter study \\
Lee--Myung~\citep{lee2022aead} & time shifts + added events & no & -- & -- \\
Yao et al.~\citep{yao2024raw} & direct raw-event attack & no constraint & -- & -- \\
Du et al.~\citep{du2026configurable} & raw-event, configurable-latency attack & no exact count constraint & -- & latency-robust optimization \\
Yu et al.~\citep{yu2026retiming} & timing-only retiming & yes & -- & -- \\
Temporal Poisoning~\citep{temporalpoison2026} & timestamp redistribution & yes & collapsed counts & detector study \\
\textbf{This work} & exact-null retiming & yes & \textbf{exact chosen $R$} & \textbf{observer-relative metric + cross-observer tests} \\
\bottomrule
\end{tabular}
\end{table}

Table~\ref{tab:related} separates three axes that are easy to conflate. Yu et al. study timing-only, count-preserving retiming, while Temporal Poisoning uses a count-preserving timestamp transformation in a training-time backdoor setting. Our test-time setting instead declares the protected representation $R$, measures success subject to $R(X')=R(X)$ or a tolerance around it, and then changes the observer to test whether the same perturbation remains hidden.

\FloatBarrier
\subsection{Transfer-function visualization}
\label{app:visibilityfig}
Figure~\ref{fig:visibility} visualizes the continuous counterpart of the discrete blind-space argument. At boxcar harmonics, the accumulated-frame response is zero while the idealized LIF response remains nonzero. Near a boxcar null, a small frequency mismatch restores frame visibility according to Eq.~(\plaineqref{eq:rho}). This visualization is explanatory; the recorded-event attacks use the exact integer guarantee of Proposition~\ref{prop:discrete-null}.
\begin{figure}[!htbp]
\centering
\includegraphics[width=.72\linewidth]{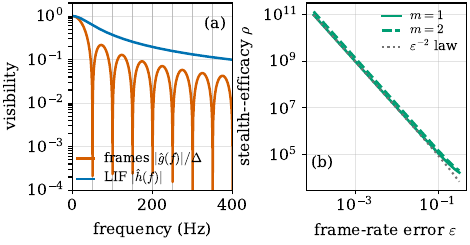}
\caption{Visibility of one temporal modulation to two consumers. Boxcar accumulation has zeros at frame-rate harmonics. The LIF membrane has nonzero gain. The right panel shows the $\eps^{-2}$ growth of the stealth--efficacy ratio near a boxcar null.}
\label{fig:visibility}
\end{figure}

\FloatBarrier
\section{Experimental details}
\label{app:expdetails}

\FloatBarrier
\subsection{N-MNIST protocol}
We parse the original binary N-MNIST files directly. Events are binned into $150$ steps of $2$ ms over a $300$ ms recording. Two polarities are retained separately on the $34\times34$ sensor grid. The coarse representation sums every $10$ fine bins, producing $15$ windows. We train three seeds for a study-specific frame CNN, a leaky integrate-and-fire (LIF) MLP~\citep{gerstner2014neuronal}, and a temporal GRU~\citep{cho2014gru}. The headline attack uses $1{,}000$ balanced test samples ($100$ per digit) and unique-event budgets $\{5,10,20,30\}\%$.

The LIF model uses an exponential decay corresponding to $\tau=4$ ms. The strict unique-event attack maintains an availability tensor containing only original events that have never been moved. Four gradient refresh rounds are used. The random control uses the same unique-event budget and samples a legal destination uniformly inside the original coarse window.

\begin{table}[!htbp]
\caption{N-MNIST ASR (\%) over three seeds.}
\label{tab:nmnist-full}
\centering\small
\begin{tabular}{llrrrr}
\toprule
Victim & Attack & 5\% & 10\% & 20\% & 30\%\\
\midrule
LIF & gradient retiming & 8.66$\pm$0.15 & 23.12$\pm$0.98 & 54.81$\pm$1.77 & 75.92$\pm$1.71\\
LIF & random retiming & 0.20$\pm$0.10 & 0.17$\pm$0.12 & 0.34$\pm$0.21 & 0.55$\pm$0.26\\
GRU & gradient retiming & 66.08$\pm$2.45 & 90.70$\pm$1.64 & 98.44$\pm$1.05 & 99.45$\pm$0.26\\
GRU & random retiming & 0.17$\pm$0.12 & 0.24$\pm$0.15 & 0.37$\pm$0.06 & 0.41$\pm$0.20\\
\bottomrule
\end{tabular}
\end{table}

The mean adversarial absolute displacement decreases mildly with budget because later moves use lower-gain destinations. For LIF it is $11.71$, $11.16$, $10.54$, and $10.10$ ms across the four budgets. For the GRU it is $10.54$, $9.97$, $9.50$, and $9.23$ ms. The random baseline has mean displacement about $7.33$ ms. This N-MNIST baseline matches the event budget but not the displacement distribution. We therefore do not use it for the event-identity isolation claim. On DVS Gesture, exact signed-displacement matching is used for the GRU and Transformer event-identity isolation claim. ConvSNN and SEW use displacement-conditioned controls. CIFAR10-DVS, N-Caltech101, and DailyDVS-200 use the exact-control protocol reported in their respective sections.

The raw-stream export routine materializes count-space retiming back into timestamp-edited event lists while preserving event count, $x$, $y$, and polarity. Three saved examples pass exact reconstruction checks against the attacked count tensor.

\FloatBarrier
\subsection{DVS Gesture protocol}
Each DVS128 Gesture segment is stored as an $N\times4$ event array $(x,y,p,t)$, and train and test subject IDs do not overlap. Sensor coordinates remain $128\times128$ in the source and are deterministically pooled by $2\times2$ blocks to $64\times64$ for model input. Each gesture duration $D_i$ is normalized to $160$ temporal bins, so the physical width of a fine bin is $D_i/160$. Eight fine bins form one coarse frame.

Three seeds are trained for a frame ResNet-18-style model~\citep{he2016resnet}, our compact ConvSNN, a SEW-ResNet-18-style SNN~\citep{fang2021sew}, a temporal GRU~\citep{cho2014gru}, our Event Transformer using a standard Transformer encoder~\citep{vaswani2017attention}, and our CoarseFrameFormer. Main attack budgets are $\{2,5,10,20\}\%$. The full test attack set contains $264$ samples. The SEW implementation is a compact PyTorch reimplementation of the element-wise residual principle, not the official reference code. The temporal GRU, Event Transformer, and CoarseFrameFormer use the same cached event tensors and subject-disjoint split.

\begin{table}[!htbp]
\caption{DVS Gesture ASR (\%) over three seeds. ``Disp.-cond. random'' samples attack-derived absolute shift magnitudes, then chooses random legal event identities/directions.}
\label{tab:dvs-full}
\centering\small
\begin{tabular}{llrrrr}
\toprule
Victim & Attack & 2\% & 5\% & 10\% & 20\%\\
\midrule
ConvSNN & gradient retiming & 24.98$\pm$3.43 & 56.45$\pm$5.26 & 81.56$\pm$5.81 & 90.98$\pm$1.27\\
ConvSNN & disp.-cond. random & 0.14$\pm$0.25 & 0.14$\pm$0.24 & 0.56$\pm$0.49 & 1.70$\pm$0.75\\
SEW-ResNet & gradient retiming & 10.44$\pm$3.75 & 23.22$\pm$3.74 & 34.06$\pm$3.69 & 48.17$\pm$1.40\\
SEW-ResNet & disp.-cond. random & 0.27$\pm$0.23 & 0.55$\pm$0.23 & 0.82$\pm$0.70 & 1.77$\pm$0.92\\
\bottomrule
\end{tabular}
\end{table}

\begin{figure}[!htbp]
\centering
\includegraphics[width=.80\linewidth]{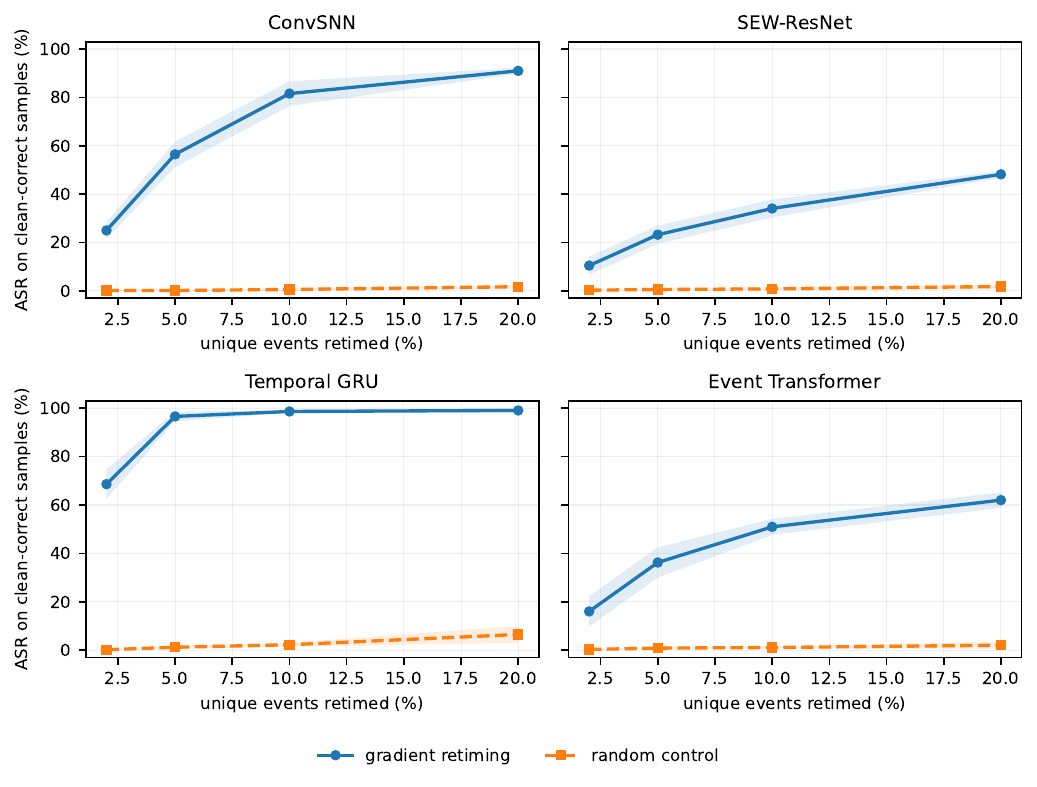}
\caption{DVS Gesture ASR versus the fraction of unique original events retimed. Solid curves show optimized null-space retiming and dashed curves show random controls. GRU and Transformer controls reproduce each sample's signed-shift multiset exactly. ConvSNN and SEW use displacement-conditioned controls. Lines are means over three seeds and bands show $\pm1$ standard deviation.}
\label{fig:real-asr}
\end{figure}

\FloatBarrier
\subsection{Exact displacement-matched random control}
\label{app:controls}
For the temporal GRU and Event Transformer, we use an exact signed-displacement matcher. Let the optimized attack on a sample produce signed fine-bin shifts $\{d_i\}_{i=1}^M$. The control randomly assigns these same shifts to distinct original events for which $s_i+d_i$ remains inside the protected coarse window. The random and adversarial perturbations therefore have the same event budget, signed displacement histogram, mean absolute displacement, maximum displacement, and coarse-window support. Only the selected event identities differ.

Verification is exhaustive over the saved evaluation. For each of the temporal GRU and Event Transformer victims, all $264$ DVS Gesture test samples are checked at four budgets and three seeds. This gives $2\times264\times4\times3=6{,}336$ sample-level matched controls. Every signed histogram matches exactly and the fallback count is zero. Table~\ref{tab:dvs-addon} reports the resulting ASR.

\begin{table}[!htbp]
\caption{DVS Gesture ASR (\%, mean$\pm$std over 3 seeds) for the recurrent and attention-based victims. ``Exact random'' reproduces the adversarial signed-shift multiset sample by sample.}
\label{tab:dvs-addon}
\centering\small
\begin{tabular}{llrrrr}
\toprule
Victim & Attack & 2\% & 5\% & 10\% & 20\%\\
\midrule
Event Transformer & gradient & 16.00$\pm$6.30 & 36.23$\pm$6.22 & 50.96$\pm$3.27 & 62.01$\pm$3.21\\
Event Transformer & exact random & 0.27$\pm$0.23 & 0.81$\pm$0.40 & 1.08$\pm$0.45 & 2.01$\pm$1.37\\
Temporal GRU & gradient & 68.61$\pm$6.07 & 96.63$\pm$2.01 & 98.67$\pm$0.45 & 99.11$\pm$0.90\\
Temporal GRU & exact random & 0.15$\pm$0.26 & 1.18$\pm$0.66 & 2.21$\pm$0.84 & 6.46$\pm$3.47\\
\bottomrule
\end{tabular}
\end{table}

\paragraph{Stronger coarse consumer.}
CoarseFrameFormer operates only on the $20$ aggregated DVS Gesture frames and obtains clean accuracies $93.56\%$, $93.94\%$, and $93.94\%$ over seeds 0--2 ($93.81\pm0.22\%$). During the GRU and Transformer attack runs, both the frame ResNet and CoarseFrameFormer are evaluated on every constrained adversarial tensor. For every seed and budget, their logits are bitwise identical to clean logits and their prediction flip rates are zero.

\FloatBarrier
\subsection{Maximum-shift ablation}
\label{app:shift}
The maximum-shift ablation uses seed 0, the full DVS Gesture test set, a $10\%$ unique-event budget, and the same progressive $2\to5\to10\%$ optimization path as the main attack. Figure~\ref{fig:shift-ablation} reports the attack under restricted temporal movement. Every condition retains exact coarse-frame equality. The unrestricted condition agrees exactly with the main seed-0 result for ConvSNN, SEW, and GRU. The Transformer differs by $1.22$ percentage points. We therefore flag that consistency check and do not use the ablation to validate its main-table value.

\begin{figure}[!htbp]
\centering
\includegraphics[width=.72\linewidth]{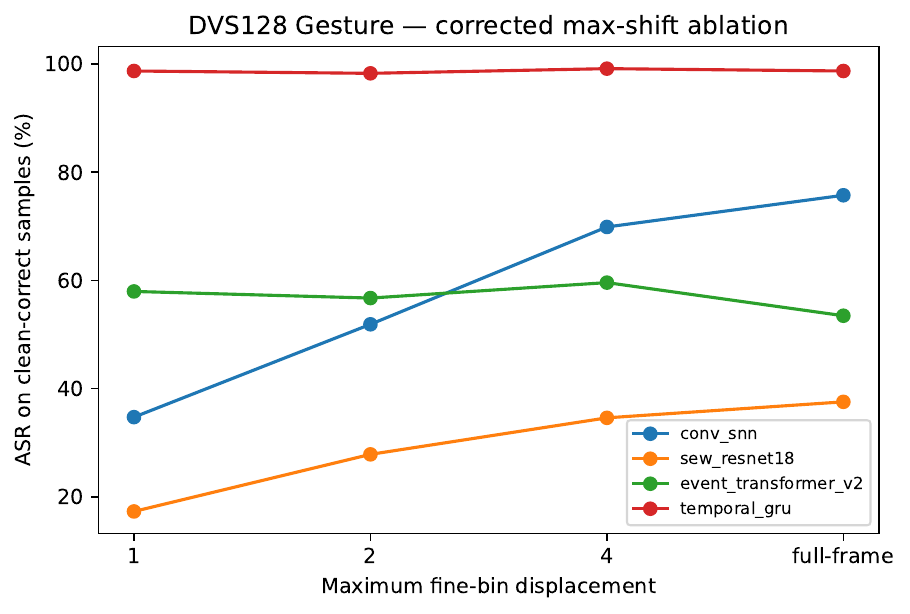}
\caption{DVS Gesture maximum-shift ablation at a $10\%$ unique-event budget (seed 0) using the same progressive path as the main attack. One fine bin corresponds to a sample-dependent physical duration. The mean absolute displacement in the 1-bin condition is $42.46$ ms.}
\label{fig:shift-ablation}
\end{figure}

\begin{table}[!htbp]
\caption{DVS Gesture shift ablation, seed 0. ASR (\%) under the main progressive path.}
\centering\small
\begin{tabular}{lrrrr}
\toprule
Victim & 1 bin & 2 bins & 4 bins & full window\\
\midrule
ConvSNN & 34.73 & 51.88 & 69.87 & 75.73\\
SEW-ResNet & 17.30 & 27.85 & 34.60 & 37.55\\
Event Transformer & 57.96 & 56.73 & 59.59 & 53.47\\
Temporal GRU & 98.70 & 98.27 & 99.13 & 98.70\\
\bottomrule
\end{tabular}
\end{table}

\FloatBarrier
\subsection{Cost-of-stealth diagnostics}
The cost-of-stealth studies are smaller single-seed experiments. The N-MNIST comparison uses four gradient rounds for both arms on the same $100$ samples at a $20\%$ event budget. For the LIF victim, exact-null and free retiming reach $50\%$ and $100\%$ ASR. For the temporal GRU both reach $100\%$ ASR. The exact-null arm has zero coarse-frame difference in both cases, while free retiming reaches maximum integer differences $14$ and $15$.

The CIFAR10-DVS comparison uses the same progressive $1\to2\to5\to10\to20\%$ path for both arms on $250$ samples. At $10\%$, null/free ASR is $96.6/98.9\%$ for ConvSNN, $71.8/85.3\%$ for SEW, $88.1/79.7\%$ for Event Transformer, and $100.0/95.7\%$ for the GRU. Every null-space row has zero coarse difference and bitwise-identical CoarseFrameFormer logits. Free retiming changes the protected prediction on $28.0$--$44.4\%$ of these samples and has maximum integer coarse-frame differences $161$--$192$. Removing the null constraint does not uniformly improve attack efficacy across architectures. In every case, however, it removes the exact protected-representation guarantee in this diagnostic.

\FloatBarrier
\subsection{Attention-based and recurrent victims}
\label{app:strongtransformer}
The Event Transformer, named \texttt{event\_transformer\_v2} in the released code, combines a residual GroupNorm spatial encoder, token projection to $384$ dimensions, a depthwise/pointwise temporal convolutional mixing stage, and an eight-layer, eight-head Transformer encoder. Its clean test accuracies are $92.80\%$, $94.70\%$, and $93.56\%$ over the three seeds ($93.69\pm0.95\%$). The attention-based victim is evaluated at $93.69\pm0.95\%$ clean accuracy, close to the strongest event models in the study.

The DVS temporal GRU provides a non-spiking sequence baseline and a cross-dataset counterpart to the N-MNIST GRU. Its clean accuracies are $87.50\%$, $84.47\%$, and $84.47\%$ ($85.48\pm1.75\%$). Its attack results are reported with the same strict unique-event budget and exact signed-displacement-matched control as the Transformer.

\FloatBarrier
\section{CIFAR10-DVS protocol}
\label{app:cifar}

We use the original CIFAR10-DVS recordings~\citep{li2017cifar10dvs} after removing AEDAT2 trigger and external-event records before DVS decoding. Deterministic $2\times2$ spatial pooling maps the $128\times128$ sensor to $64\times64$. Each recording is represented by $T=80$ fine bins with $S=8$ bins per protected window and ten coarse frames. The split uses $800$ train and $200$ test recordings per class with split seed 2027. The balanced attack subset contains $100$ test recordings per class.

The protected model is CoarseFrameFormer. Temporal victims are ConvSNN, SEW-ResNet-18 style, Event Transformer, and Temporal GRU. Budgets are $\{1,2,5,10,20\}\%$ unique original events. Legal moves remain inside the original coarse window. Controls include uniform random retiming and exact signed-displacement-matched retiming. The protocol also includes a $10\%$ maximum-shift sweep over $\{1,2,4,\mathrm{full}\}$ fine bins and a free-retiming cost-of-stealth diagnostic.

\begin{table}[!htbp]
\caption{CIFAR10-DVS clean accuracy and ASR (\%, mean$\pm$std over three seeds). Random columns use the exact signed-displacement-matched control.}
\label{tab:cifar-results}
\centering\small
\begin{tabular}{lccccc}
\toprule
Victim & Clean acc. & 5\% ASR & 5\% random & 10\% ASR & 10\% random \\
\midrule
ConvSNN & \CIFARConvClean & \CIFARConvASRFive & \CIFARConvRandFive & \CIFARConvASRTen & \CIFARConvRandTen \\
SEW-ResNet & \CIFARSEWClean & \CIFARSEWASRFive & \CIFARSEWRandFive & \CIFARSEWASRTen & \CIFARSEWRandTen \\
Temporal GRU & \CIFARGRUClean & \CIFARGRUASRFive & \CIFARGRURandFive & \CIFARGRUASRTen & \CIFARGRURandTen \\
Event Transformer & \CIFARTransformerClean & \CIFARTransformerASRFive & \CIFARTransformerRandFive & \CIFARTransformerASRTen & \CIFARTransformerRandTen \\
\bottomrule
\end{tabular}
\end{table}

All constrained attack and exact-control rows have zero integer coarse-tensor difference, bitwise-identical CoarseFrameFormer logits, and zero protected-consumer prediction flips.

\FloatBarrier
\section{N-Caltech101 protocol and results}
\label{app:ncaltech}
N-Caltech101~\citep{orchard2015nmnist} is represented with $T=80$ fine bins, eight bins per protected window, and a fixed aspect-preserving letterbox from $240\times180$ to $64\times64$. We use all 101 classes, an $80/20$ split with seed 2027, three training seeds, and a deterministic proportionally stratified attack subset of $1{,}000$ test streams. The victim registry and budgets are the same as CIFAR10-DVS. Table~\ref{tab:ncal-results} reports the complete three-seed result.

\begin{table}[!htbp]
\caption{N-Caltech101 ASR (\%, mean$\pm$std over three seeds). ``Exact'' is the signed-displacement-matched random control.}
\label{tab:ncal-results}
\centering\scriptsize
\setlength{\tabcolsep}{3.4pt}
\begin{tabular}{llrrrrr}
\toprule
Victim & Attack & 1\% & 2\% & 5\% & 10\% & 20\% \\
\midrule
ConvSNN & optimized & 12.08$\pm$0.63 & 19.81$\pm$1.69 & 30.89$\pm$2.09 & 44.48$\pm$1.57 & 60.01$\pm$1.28 \\
ConvSNN & exact & 0.42$\pm$0.09 & 0.73$\pm$0.24 & 1.25$\pm$0.32 & 1.67$\pm$0.60 & 1.57$\pm$0.16 \\
SEW-ResNet & optimized & 7.45$\pm$1.85 & 14.58$\pm$1.03 & 26.89$\pm$1.59 & 36.23$\pm$3.39 & 46.71$\pm$3.56 \\
SEW-ResNet & exact & 0.71$\pm$0.62 & 0.52$\pm$0.54 & 0.80$\pm$0.54 & 1.27$\pm$0.71 & 2.46$\pm$1.08 \\
Transformer & optimized & 26.89$\pm$2.26 & 45.96$\pm$3.36 & 60.10$\pm$2.73 & 70.66$\pm$2.83 & 80.01$\pm$2.18 \\
Transformer & exact & 0.77$\pm$0.08 & 0.88$\pm$0.09 & 1.70$\pm$0.37 & 2.25$\pm$0.35 & 5.33$\pm$0.50 \\
GRU & optimized & 35.52$\pm$1.92 & 59.05$\pm$1.63 & 79.96$\pm$3.57 & 89.48$\pm$3.34 & 93.77$\pm$2.27 \\
GRU & exact & 2.15$\pm$0.32 & 2.82$\pm$0.20 & 6.44$\pm$0.55 & 8.78$\pm$1.35 & 14.68$\pm$1.01 \\
\bottomrule
\end{tabular}
\end{table}

\FloatBarrier
\section{DailyDVS-200 protocol and results}
\label{app:dailydvs}
DailyDVS-200~\citep{wang2024dailydvs} is decoded from the released AEDAT4 recordings and letterboxed from $320\times240$ to $64\times64$. We use $T=80$, eight fine bins per protected window, all 200 classes, three seeds, and a proportionally stratified $1{,}000$-stream attack subset. The split contains 31 training, 9 test, and 7 validation subjects with no subject overlap after the released validation/test overlap is removed from validation. The train/test recording overlap is zero. Table~\ref{tab:dailydvs-results} reports the complete three-seed attack result. ConvSNN clean accuracy is only $13.55\pm0.60\%$, so its ASR is supporting rather than headline evidence. The remaining fine-time victims reach $27.31$--$34.25\%$ clean accuracy on the 200-way task.

\begin{table}[!htbp]
\caption{DailyDVS-200 ASR (\%, mean$\pm$std over three seeds). ``Exact'' is the signed-displacement-matched random control.}
\label{tab:dailydvs-results}
\centering\scriptsize
\setlength{\tabcolsep}{3.4pt}
\begin{tabular}{llrrrrr}
\toprule
Victim & Attack & 1\% & 2\% & 5\% & 10\% & 20\% \\
\midrule
ConvSNN & optimized & 34.39$\pm$5.21 & 50.30$\pm$8.38 & 70.37$\pm$5.20 & 84.29$\pm$6.23 & 94.89$\pm$2.43 \\
ConvSNN & exact & 4.75$\pm$1.60 & 5.75$\pm$0.50 & 7.15$\pm$1.31 & 9.81$\pm$2.01 & 15.81$\pm$5.14 \\
SEW-ResNet & optimized & 11.35$\pm$0.99 & 17.46$\pm$1.32 & 27.11$\pm$2.50 & 34.23$\pm$3.26 & 46.14$\pm$1.83 \\
SEW-ResNet & exact & 3.64$\pm$1.84 & 5.19$\pm$1.34 & 5.62$\pm$0.92 & 9.30$\pm$1.81 & 11.27$\pm$0.77 \\
Transformer & optimized & 65.24$\pm$1.18 & 85.50$\pm$1.07 & 94.34$\pm$1.13 & 98.27$\pm$0.96 & 99.38$\pm$0.84 \\
Transformer & exact & 3.00$\pm$1.04 & 5.12$\pm$1.52 & 8.20$\pm$1.35 & 13.65$\pm$3.18 & 24.75$\pm$1.96 \\
GRU & optimized & 66.33$\pm$1.27 & 87.24$\pm$1.24 & 95.25$\pm$1.11 & 98.88$\pm$0.36 & 99.25$\pm$0.39 \\
GRU & exact & 2.38$\pm$0.81 & 3.75$\pm$0.71 & 7.25$\pm$0.83 & 16.15$\pm$1.42 & 37.43$\pm$1.10 \\
\bottomrule
\end{tabular}
\end{table}

All optimized and exact-control DailyDVS-200 rows preserve the coarse tensor exactly. CoarseFrameFormer logits are bitwise identical and its prediction flip rate is zero in every saved constrained row.

\FloatBarrier
\section{Direct retiming comparison and representation-space metric}
\label{app:direct}
The direct comparison uses the official \citet{yu2026retiming} PIL-$L_0$, free gradient retiming, and our null-space retiming on the DVS128 Gesture ConvSNN. The full run is complete for three seeds (239, 235, and 235 clean-correct samples). For the nonnegative count tensor used here, $\lVert A(X)\rVert_1$ equals the total event-unit count, so $D_A$ in Eq.~(\plaineqref{eq:repdist}) is computed exactly from the saved coarse-frame $L_1$ difference divided by the saved event-unit count. Across seeds, raw ASR is $100.0\pm0.0\%$ for PIL-$L_0$, $95.06\pm0.23\%$ for free retiming, and $81.56\pm5.81\%$ for null-space retiming. Exact-stealth SC-ASR is $0$, $0$, and $81.56\pm5.81\%$, respectively.
For interpretation, define $\mathrm{SPR}_A(\tau)=\Pr[D_A\le\tau\mid\text{clean-correct}]$ and $\mathrm{ASR}_{\mathrm{pass}}(\tau)=\Pr[\text{success}\mid D_A\le\tau,\text{clean-correct}]$. Whenever the latter is defined, $\mathrm{SC\text{-}ASR}_A(\tau)=\mathrm{SPR}_A(\tau)\,\mathrm{ASR}_{\mathrm{pass}}(\tau)$. We also report stealth retention $\mathrm{SC\text{-}ASR}_A(\tau)/\mathrm{ASR}_{\mathrm{raw}}$ when raw ASR is nonzero. At $\tau=0$, null-space retiming has $100\%$ pass rate and retention, while PIL-$L_0$ and free retiming have zero pass rate and retention in the complete seed.

\begin{table}[!htbp]
\caption{Three-seed DVS direct comparison: SC-ASR (\%, mean$\pm$std). The raw row is unconstrained ASR.}
\label{tab:scasr3seed}
\centering\scriptsize
\setlength{\tabcolsep}{4.5pt}
\begin{tabular}{lrrr}
\toprule
$\tau$ & Yu PIL-$L_0$ & Free & Null \\
\midrule
0 & 0.0$\pm$0.0 & 0.0$\pm$0.0 & 81.56$\pm$5.81 \\
0.18 & 5.65$\pm$1.79 & 0.0$\pm$0.0 & 81.56$\pm$5.81 \\
0.19 & 23.42$\pm$1.12 & 0.99$\pm$0.49 & 81.56$\pm$5.81 \\
0.195 & 41.20$\pm$2.06 & 59.93$\pm$2.03 & 81.56$\pm$5.81 \\
0.20 & 62.34$\pm$1.49 & 95.06$\pm$0.23 & 81.56$\pm$5.81 \\
0.205 & 80.11$\pm$4.08 & 95.06$\pm$0.23 & 81.56$\pm$5.81 \\
\midrule
Raw & 100.0$\pm$0.0 & 95.06$\pm$0.23 & 81.56$\pm$5.81 \\
\bottomrule
\end{tabular}
\end{table}

\begin{table}[!htbp]
\caption{$\mathrm{SC\text{-}ASR}_A(\tau)$ (\%) over the full measured post-hoc representation-budget grid for the complete direct-comparison seed. Values in brackets are pointwise $95\%$ Wilson intervals over 239 clean-correct samples.}
\label{tab:scasr}
\centering\scriptsize
\setlength{\tabcolsep}{3.1pt}
\renewcommand{\arraystretch}{0.90}
\begin{tabular}{lrrr}
\toprule
$\tau$ & Yu PIL-$L_0$ & Free retiming & Null-space retiming \\
\midrule
0 & 0.00 [0.00,1.58] & 0.00 [0.00,1.58] & 74.90 [69.03,79.97] \\
0.01 & 0.00 [0.00,1.58] & 0.00 [0.00,1.58] & 74.90 [69.03,79.97] \\
0.05 & 0.00 [0.00,1.58] & 0.00 [0.00,1.58] & 74.90 [69.03,79.97] \\
0.1 & 0.00 [0.00,1.58] & 0.00 [0.00,1.58] & 74.90 [69.03,79.97] \\
0.15 & 0.00 [0.00,1.58] & 0.00 [0.00,1.58] & 74.90 [69.03,79.97] \\
0.17 & 0.84 [0.23,3.00] & 0.00 [0.00,1.58] & 74.90 [69.03,79.97] \\
0.18 & 5.02 [2.90,8.57] & 0.00 [0.00,1.58] & 74.90 [69.03,79.97] \\
0.19 & 23.01 [18.13,28.75] & 1.26 [0.43,3.62] & 74.90 [69.03,79.97] \\
0.195 & 39.33 [33.35,45.65] & 61.92 [55.63,67.85] & 74.90 [69.03,79.97] \\
0.2 & 62.34 [56.05,68.24] & 94.98 [91.43,97.10] & 74.90 [69.03,79.97] \\
0.205 & 80.75 [75.28,85.25] & 94.98 [91.43,97.10] & 74.90 [69.03,79.97] \\
0.21 & 92.47 [88.41,95.18] & 94.98 [91.43,97.10] & 74.90 [69.03,79.97] \\
0.22 & 96.65 [93.54,98.29] & 94.98 [91.43,97.10] & 74.90 [69.03,79.97] \\
0.23 & 99.58 [97.67,99.93] & 94.98 [91.43,97.10] & 74.90 [69.03,79.97] \\
0.24 & 100.00 [98.42,100.00] & 94.98 [91.43,97.10] & 74.90 [69.03,79.97] \\
\bottomrule
\end{tabular}
\end{table}

The seed-0 table above gives Wilson intervals. Table~\ref{tab:scasr3seed} is the completed three-seed headline comparison. The strict point $\tau=0$ is one operating point of the monitor. The complete grid shows how much realized attack efficacy remains as the allowed representation distortion is relaxed. The three-seed raw ASRs are $100.0\%$, $95.06\pm0.23\%$, and $81.56\pm5.81\%$ for PIL-$L_0$, free retiming, and null-space retiming. Their exact-stealth SC-ASRs are $0$, $0$, and $81.56\pm5.81\%$.

\FloatBarrier
\subsection{Temporal locality of the retiming attacks}
\label{app:direct-eps}
The representation budget $\tau$ constrains what is visible to the protected observer but does not constrain how far events move in time. We also report
\[
A(\epsilon)=\Pr[\text{attack succeeds and }d_t(X,X')\le\epsilon\mid\text{clean-correct}],
\]
where $d_t$ is mean absolute timestamp displacement. The unified runner contains five attacks, but the temporal-radius analysis covers Yu PIL-$L_0$, free retiming, and Null. At the seed-0 10\% DVS operating point, mean displacement is $276.9$ ms for Yu, $2.565$ s for free retiming, and $179.5$ ms for Null. Table~\ref{tab:epsball} reports selected radii. Adapted SDA/PDSG and Yao/Gumbel remain in the five-method $\mathrm{SC\text{-}ASR}_A(\tau)$ evaluation but are not included in this completed temporal-radius diagnostic. These values supersede displacement values from an earlier direct-comparison implementation. The representation-space $\tau$ results are unaffected because they are recomputed from the protected tensors.

\begin{table}[!htbp]
\caption{Temporal-radius success for the three timestamp-retiming methods in the unified DVS seed-0 run at the 10\% operating point. Values are $A(\epsilon)$ in percent.}
\label{tab:epsball}
\centering\scriptsize
\setlength{\tabcolsep}{4.2pt}
\begin{tabular}{lrrr}
\toprule
$\epsilon$ (ms) & Yu PIL-$L_0$ & Free retiming & Null-space retiming \\
\midrule
100  & 0.42 & 0.00 & 1.67 \\
200  & 7.11 & 0.00 & 51.88 \\
300  & 68.20 & 0.00 & 74.48 \\
404  & 96.65 & 0.00 & 74.48 \\
600  & 99.58 & 0.00 & 74.90 \\
1000 & 100.00 & 0.42 & 74.90 \\
3000 & 100.00 & 74.48 & 74.90 \\
\bottomrule
\end{tabular}
\end{table}

The $\tau$ and $\epsilon$ analyses answer different questions. The $\tau$ axis measures observer-relative stealth, while $\epsilon$ measures temporal movement. In the unified run, Null remains exactly invariant and requires less temporal movement than free retiming. Yu achieves high temporal-radius success but does not satisfy exact protected equality.

\FloatBarrier
\subsection{Null--C-PGD comparison details}
\label{app:fairness}
The direct Null--C-PGD comparison uses the same victim family and exact-null feasible set. Table~\ref{tab:fairness-audit} separates this comparison from the heterogeneous five-attack evaluation.
\begin{table}[!htbp]
\caption{Threat-model comparison. ``Same'' means identical between Null and C-PGD in the direct exact-null experiment. Prior attacks retain their native constraints and are therefore evaluated post hoc rather than treated as a common feasible-set contest.}
\label{tab:fairness-audit}
\centering\scriptsize
\setlength{\tabcolsep}{2.2pt}
\begin{tabular}{lcccccc}
\toprule
Method & timestamp-only & $x/y/p$ fixed & exact $A$ & event budget & displacement rule & role\\
\midrule
Null & yes & yes & yes & 10\% upper & within window & exact-null optimizer\\
C-PGD & yes & yes & yes & 10\% upper & within window & exact-null optimizer\\
Free & yes & yes & no & $\approx$10\% & native/free & post-hoc evaluation\\
Yu PIL-$L_0$ & yes & yes & no chosen-$A$ constraint & $\approx$10\% realized & native PIL constraints & post-hoc evaluation\\
SDA / Yao & native attack & native attack & no chosen-$A$ constraint & native realized & native & post-hoc evaluation\\
\bottomrule
\end{tabular}
\end{table}
Null and C-PGD use identical attack populations and clean-correct masks within each dataset and seed. DVS uses the 264-clip released test population with 239/235/235 clean-correct clips across seeds. DailyDVS uses the frozen 1,000-stream attack subset with 116/134/131 clean-correct clips. The matched comparison uses three victim-gradient evaluations on DVS and four on DailyDVS. C-PGD uses continuous projected-gradient updates followed by row-simplex projection and integer transport; Null uses discrete greedy move ranking.

The two DailyDVS ConvSNN Null values reported in the paper use the same frozen 1,000-clip attack subset, the same checkpoints, and identical clean-correct sample UIDs: 116/134/131 clips across the three seeds. The broad experiment reports $84.29\pm6.23\%$ at 10\% from a single-pass budget path. The direct Null--C-PGD experiment reports $86.98\pm4.44\%$ after the progressive $1\%\to2\%\to5\%\to10\%$ path with four gradient refreshes, with the defense pipeline enabled. The 2.69-point difference therefore comes from the attack schedule rather than a different evaluation population. The saved manifest is \texttt{attack\_subset\_manifest.csv} (SHA-256 prefix \texttt{7bf9a4c4}).

\FloatBarrier
\subsection{Statistical tests}
\label{app:statistics}
For repeated-clip, repeated-seed results, we use a 20,000-replicate clip-clustered, seed-aware bootstrap. Each replicate resamples original evaluation clip IDs, computes ASR separately within each seed using that seed's clean-correct mask, and then averages the seed rates. Paired method differences use the same resampled clips. On DVS Gesture, Null has 81.56\% exact-null ASR with 95\% CI $[77.24,85.63]\%$, while matched-gradient C-PGD has 84.50\% with $[80.33,88.45]\%$. The paired C-PGD$-$Null difference is 2.94 points with $[0.85,5.12]$. On DailyDVS-200, the corresponding values are 86.98\% $[82.32,91.38]$, 89.55\% $[85.14,93.60]$, and a 2.57-point difference with $[-0.29,5.56]$. Mean$\pm$std across the three training seeds is retained as a descriptive summary. Pooled Wilson intervals remain in saved single-population tables but are not used as primary across-seed uncertainty.

Paired McNemar tests on the complete original direct-comparison seed confirm the raw-ASR differences of free retiming and PIL-$L_0$ versus null-space retiming ($p=5.6\times10^{-10}$ and $1.7\times10^{-18}$). Across CIFAR10-DVS, N-Caltech101, and DailyDVS-200, all $12$ optimized-versus-exact-control comparisons remain significant after Holm correction. The largest adjusted value is $p=0.0119$.

\FloatBarrier
\subsection{Attack-generation cost}
\label{app:cost}
On an NVIDIA H100 80GB, the full DailyDVS-200 progressive $1\to2\to5\to10\to20\%$ attack path takes $467\pm1$ s for ConvSNN, $1490\pm41$ s for SEW-ResNet, $272\pm9$ s for the Event Transformer, and $315\pm10$ s for the GRU across three seeds. These times cover attack generation on the same $1{,}000$-stream subset and exclude model training and preprocessing. The current implementation uses dense fine-time tensors, so its cost increases with temporal and spatial resolution.

For recorded DVS streams, randomized schedules are evaluated on fixed canonical-null attacks. About 97.5\% of realized schedules expose a perturbation that is invisible to the canonical accumulator. Because the victim is not retrained under randomized accumulation, this experiment measures observer visibility rather than complete randomized-consumer robustness. Appendix~\ref{app:controlled} separately evaluates a schedule-aware detector.

\FloatBarrier
\section{Controlled spectral validation}
\label{app:controlled}
The recorded-event experiments establish exact integer invariance on event streams. A controlled study separately isolates the spectral mechanism and randomized accumulation under the assumptions of Section~\ref{sec:theory}. We generate saccade-derived activity over $60$ substeps of $2$ ms, aggregate every ten substeps ($\Dt=20$ ms), and compare a frame CNN with an LIF network at $\tau=4$ ms. The activity tensor in this controlled study is continuous-valued. We use this continuous-valued study only to validate the mechanism. The headline N-MNIST, DVS Gesture, CIFAR10-DVS, N-Caltech101, and DailyDVS-200 experiments use discrete recorded events.

\FloatBarrier
\subsection{Invisible-subspace attack and the price of the constraint}
For the boxcar observer, the discrete invisible subspace contains all perturbations whose fine-time entries sum to zero inside each coarse window. A sinusoidal null carrier is a structured Fourier direction in this space. A white-box null-PGD attack instead optimizes the victim loss and projects each iterate back into the zero-sum subspace. Table~\ref{tab:controlled-attack} gives the complete controlled attack comparison.

\begin{table}[!htbp]
\caption{Controlled attack study (mean$\pm$std over 3 seeds, $1{,}000$ streams/seed). $\delta_F$ is the relative change of the accumulated frame representation. ``Flips'' is the percentage of frame-CNN labels that change. Rows labeled null lie in the protected aggregation null space.}
\label{tab:controlled-attack}
\centering\scriptsize
\setlength{\tabcolsep}{3.8pt}
\begin{tabular}{lrrrr}
\toprule
Attack & LIF acc. (\%) & frame-CNN acc. (\%) & $\delta_F$ & flips \\
\midrule
none & $96.90\pm0.08$ & $97.80\pm0.43$ & -- & -- \\
null tone, $a=2$ & $96.50\pm0.29$ & $97.80\pm0.43$ & $2\times10^{-7}$ & 0 \\
null carrier, $a=2$ & $96.57\pm0.33$ & $97.80\pm0.43$ & $2\times10^{-7}$ & 0 \\
null-PGD, $\epsilon=.10$ & $50.73\pm0.53$ & $97.80\pm0.43$ & $3\times10^{-7}$ & 0 \\
null-PGD, $\epsilon=.25$ & $0.33\pm0.12$ & $97.80\pm0.43$ & $3\times10^{-7}$ & 0 \\
null-PGD, $\epsilon=.50$ & $0.00\pm0.00$ & $97.80\pm0.43$ & $3\times10^{-7}$ & 0 \\
free PGD, $\epsilon=.10$ & $5.43\pm0.57$ & $95.80\pm0.22$ & $0.08$ & $2.1\%$ \\
free PGD, $\epsilon=.25$ & $0.00\pm0.00$ & $90.97\pm0.52$ & $0.19$ & $7.0\%$ \\
\bottomrule
\end{tabular}
\end{table}

Across the $21$ invisible-subspace configurations ($7$ settings $\times$ $3$ seeds, $1{,}000$ streams each), none of the $21{,}000$ attacked streams changes a frame-CNN label. The largest relative logit change is $2\times10^{-7}$. At $\epsilon=.25$, the constrained and free attacks both saturate the fine-time victim, but the free attack has a frame footprint nearly six orders of magnitude larger. At $\epsilon=.10$, the constraint has a measurable efficacy cost, consistent with the smaller real-event diagnostics in Section~\ref{sec:experiments}.

\FloatBarrier
\subsection{Transfer-function mechanism and synchronization error}
At $\tau/\Dt=0.2$, Corollary~\ref{cor:harmonics} predicts LIF gain $|\widehat h(1/\Dt)|\approx0.62$ at a frequency the frame accumulator rejects exactly. Sweeping $\tau/\Dt\in\{0.1,0.2,0.35,0.5,0.75,1\}$ at sub-saturating budgets yields Pearson $r=0.99$ between predicted gain and measured accuracy drop. This dependence vanishes only after the attack saturates the classifier, which is why the mechanism sweep uses smaller budgets.

A Poisson testbed with $\lambda_0=2000\,\mathrm{s}^{-1}$ and $\Dt=10$ ms directly checks Theorem~\ref{thm:poisson}. With a null-frequency modulation of amplitude $0.5$, a two-sample Kolmogorov--Smirnov test over $2\times10^4$ frames gives $p=0.57$ and variance ratio $1.01$. At a non-null frequency the same-amplitude modulation gives $p<10^{-3}$ and variance ratio $1.83$. The integrated rate per null window is unchanged to $2\times10^{-11}$. Eq.~(\plaineqref{eq:rho}) matches the exact stealth--efficacy ratio to $2.0\%$ at relative frequency error $10^{-2}$ and $0.2\%$ at $10^{-3}$. The corresponding ratios are approximately $3.9\times10^3$ and $3.9\times10^5$.

\FloatBarrier
\subsection{Randomized accumulation and schedule-aware detection}
\begin{figure}[!htbp]
\centering
\includegraphics[width=.36\linewidth]{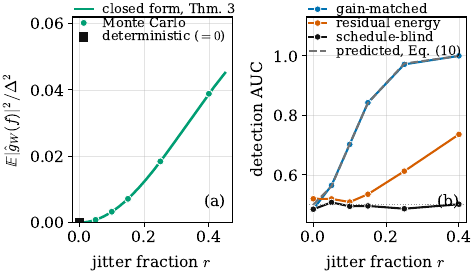}
\caption{Randomized accumulation: expected response energy (left) and schedule-blind, residual-energy, and schedule-aware detection (right).}
\label{fig:defense-main}
\end{figure}

Figure~\ref{fig:defense-main} validates Theorem~\ref{thm:randomized} and Eq.~(\plaineqref{eq:deflection}). Monte Carlo estimates of $\E|\widehat g_W(f)|^2$ agree with the closed form to within $0.25\%$. Over $400$-frame streams at attack amplitude $\alpha=0.4$, a schedule-blind detector remains at chance ($0.49$--$0.50$ AUROC) across the jitter sweep. The residual-energy detector reaches AUROC $0.61$ and $0.74$ at jitter $r=.25$ and $.40$, while the schedule-aware gain-matched detector reaches AUROC $0.97$ and $0.999$. The deflection-based prediction matches the gain-matched AUROC across all $18$ evaluated settings with mean absolute error $0.011$ (worst case $0.03$).

\FloatBarrier
\subsection{Representation refinements and scope}
The location of the null depends on the consumer's temporal representation. At the first null of a $\Dt$-wide boxcar, splitting the interval into $B$ equal sub-bins changes normalized carrier visibility to $|\mathrm{sinc}(1/B)|$. It is $0$ for $B=1$, about $0.64$ for $B=2$, and $0.98$ for $B=10$. Finer voxelization closes this blind spot but introduces higher-frequency nulls, where the LIF response is more attenuated. Eq.~(\plaineqref{eq:rho}) quantifies the trade. An unsigned activity statistic $\int g\,d|N|$ need not share the same polarity-summed null. For randomized widths, the positivity claim requires a distribution that breaks null-preserving lattices, and detection requires access to the realized schedule.

\FloatBarrier
\section{Aligned N-MNIST five-model protocol}
\label{app:nmnist-aligned}
To separate the null-space property from the older N-MNIST model choices, we reran N-MNIST with the same five-model family used in the larger event benchmarks, $T=150$ fine bins and $S=10$ bins per protected window. Clean test accuracy is approximately $99\%$ for CoarseFrameFormer, ConvSNN, SEW-ResNet18, and Event Transformer and $98.5\%$ for the GRU. Every optimized attack row has $D_A=D_\infty=0$.

\begin{table}[!htbp]
\caption{Aligned N-MNIST null-space ASR (\%, mean$\pm$std over three seeds).}
\centering\scriptsize
\begin{tabular}{lrrr}
\toprule
Victim & 10\% & 20\% & 30\% \\
\midrule
ConvSNN & 4.11$\pm$0.32 & 10.00$\pm$0.51 & 18.71$\pm$1.67 \\
SEW-ResNet18 & 1.45$\pm$0.51 & 3.50$\pm$0.24 & 4.81$\pm$0.32 \\
Event Transformer & 2.89$\pm$0.31 & 8.68$\pm$0.98 & 19.61$\pm$4.42 \\
Temporal GRU & 61.23$\pm$4.16 & 90.59$\pm$1.74 & 97.87$\pm$1.19 \\
\bottomrule
\end{tabular}
\end{table}
A max-shift ablation on seed 0 shows that the GRU already reaches $48.1\%$ ASR when every moved event is restricted to one fine bin (2 ms), compared with $56.4\%$ for unrestricted within-window moves at the same $10\%$ budget. This supports a local-timing mechanism rather than a dependence on large timestamp displacements.

\FloatBarrier
\section{DailyDVS-200 architecture stress test}
\label{app:daily-sota4}
We also evaluate protocol-scale variants of ACTION-Net~\citep{wang2021action}, MVFNet~\citep{wu2021mvfnet}, Swin-T~\citep{liu2021swin}, and TimeSformer~\citep{bertasius2021timesformer}. These models are trained from scratch under our common DailyDVS preprocessing. They are \emph{not} released benchmark checkpoints, and published benchmark scores are used only as architectural context. Table~\ref{tab:daily-sota-audit} reports the measured clean accuracy and clean-correct denominators needed to interpret conditional ASR.
\begin{table}[!htbp]
\caption{DailyDVS protocol-scale architecture results at the 10\% exact-null operating point. These are models trained under our common preprocessing, not official benchmark checkpoints. Clean Top-1 is measured on our exact evaluation split. $n_{cc}$ lists clean-correct attacked clips for seeds 0/1/2. All optimized and matched-control rows preserve the protected tensor exactly and give bit-identical protected-consumer logits.}
\label{tab:daily-sota-audit}
\centering\scriptsize
\setlength{\tabcolsep}{3pt}
\begin{tabular}{lrrrr}
\toprule
Victim & Clean Top-1 & $n_{cc}$ (s0/s1/s2) & Null ASR & matched control\\
\midrule
ActionNet & 42.18\% & 412/421/434 & 97.55$\pm$0.28 & 12.27$\pm$2.26\\
MVFNet & 50.21\% & 506/522/507 & 99.28$\pm$0.11 & 9.70$\pm$1.06\\
Swin-T & 22.74\% & 211/247/194 & 99.66$\pm$0.60 & 20.01$\pm$2.94\\
TimeSformer & 37.63\% & 367/330/361 & 97.94$\pm$0.61 & 9.74$\pm$2.61\\
\bottomrule
\end{tabular}
\end{table}
We use this result as a model-strength and architecture stress test under a controlled common representation. We do not claim that our training reproduces the published DailyDVS leaderboard. We report the low Swin-T clean accuracy explicitly. ASR remains conditioned on clean-correct clips for every model.

\FloatBarrier
\section{SHD cross-modality extension}
\label{app:shd}

SHD~\citep{cramer2020heidelberg} is a neuromorphic auditory benchmark and
is used only to test whether the discrete null-space mechanism extends
beyond event vision. We use the same $T=80,S=8$ construction, and all
constrained rows satisfy exact protected equality.

\begin{table}[t]
\caption{\textbf{Cross-modality extension to SHD.}
Three-seed ASR (\%, mean$\pm$std) under exact protected equality.}
\label{tab:shd}
\centering
\small
\setlength{\tabcolsep}{7pt}
\begin{tabular}{lcc}
\toprule
Victim & 10\% budget & 20\% budget \\
\midrule
ConvSNN     & $53.9\pm3.4$ & $74.2\pm3.7$ \\
SEW         & $12.9\pm1.1$ & $23.9\pm1.2$ \\
Transformer & $49.5\pm4.3$ & $67.1\pm2.6$ \\
GRU         & $59.0\pm2.6$ & $73.3\pm1.6$ \\
\bottomrule
\end{tabular}
\end{table}

The SHD result shows the same exact observer invisibility outside event vision, while
vulnerability within the blind space remains architecture dependent.

\FloatBarrier
\section{Unified five-attack and operational-defense details}
\label{app:operational-partial}
The unified DVS run evaluates Yu PIL-$L_0$, Free, Null, adapted SDA/PDSG, and adapted Yao/Gumbel on the same ConvSNN/CoarseFrameFormer population. Seed-0 headline values are reproduced below.
\begin{table}[!htbp]
\caption{Unified DVS seed-0 attack/defense result.}
\centering\scriptsize
\begin{tabular}{lrrrrr}
\toprule
Method & ASR & mean $D_A$ & protected flip & AUROC & UASR@5 \\
\midrule
Null & 74.90 & 0.000 & 0.0 & 0.493 & 73.64 \\
SDA  & 40.59 & 0.010 & 0.0 & 0.494 & 40.59 \\
Free & 94.56 & 0.194 & 3.77 & 0.509 & 92.05 \\
Yu PIL-$L_0$ & 100.00 & 0.198 & 2.51 & 0.432 & 98.74 \\
Yao & 64.02 & 0.523 & 34.73 & 0.932 & 35.56 \\
\bottomrule
\end{tabular}
\end{table}
The clean-only detector is fit on held-out clean clips and evaluated on the attack population. The realized clean FPR for the combined detector is $2.51\%$ for the main DVS rows at the nominal 5\% operating point (Yao: $3.35\%$ under its fold-level calibration). We report the realized FPR together with nominal operating points in the saved tables.

Across the completed three-seed DVS evaluation, Yu/SDA/Yao/Free/Null have AUROC
$0.431/0.495/0.929/0.510/0.492$ and UASR@5
$98.6/34.9/36.1/92.4/79.9\%$, respectively. These aggregate values are the
operational numbers reported in Section~\ref{sec:operational}. The table above
is retained as the seed-0 view.

A controlled Free sweep holds the attack family and approximate 10\% edit footprint fixed while changing only the maximum allowed shift. Mean $D_A$ progresses $0.0418\to0.0772\to0.1422\to0.1759\to0.1889$ for max-shift $2,4,8,16,32$ fine bins. ASR progresses $37.7\to54.4\to64.9\to72.8\to74.1\%$. Detector true-positive rate (TPR) at fixed false-positive rate (FPR) does not increase monotonically. Representation distance therefore does not provide a universal measure of operational detectability in this sweep.

The available DailyDVS unified operational diagnostic contains 119 pooled clean-correct evaluations across three seeds, substantially fewer than the broad 1,000-stream protocol. It shows the same non-monotonic pattern (Null $D_A=0$, Yu/Free around $0.2$, Yao around $0.63$). All but Yao remain near chance on DVS, and all methods are near chance on this smaller Daily reporting population. We treat this as an underpowered diagnostic and do not use it for a main cross-dataset detection claim.

\FloatBarrier
\section{Observer-family ablations}
\label{app:observer-family}
\label{app:ablations}
Table~\ref{tab:observer-family-ablation} summarizes the DVS observer-family experiments. The fixed-attack rows apply the same Null streams to different temporal views. The intersection rows re-optimize the attack while enforcing equality to additional phase observers.

\begin{table}[!htbp]
\caption{DVS observer-family ablations. Top: mean relative representation change for the same canonical-null ConvSNN attacks ($n=239$). Bottom: blind-space size and ASR when the attack is constrained to multiple phase observers.}
\label{tab:observer-family-ablation}
\centering\scriptsize
\setlength{\tabcolsep}{4pt}
\begin{tabular}{lrrr}
\toprule
\multicolumn{4}{l}{\textit{Fixed attack, changed observer}}\\
Observer & \multicolumn{3}{c}{Mean relative change}\\
\midrule
Canonical width-8, phase 0 & \multicolumn{3}{c}{0.0000}\\
Width-8, phase 4 & \multicolumn{3}{c}{0.1357}\\
Width-8, stride-4 overlap & \multicolumn{3}{c}{0.1291}\\
4/8/16-bin multiscale & \multicolumn{3}{c}{0.1435}\\
\midrule
\multicolumn{4}{l}{\textit{Attack constrained to observer intersections}}\\
Observers & Blind dim. (fraction) & ConvSNN ASR & SEW ASR\\
\midrule
Phase 0 & 1,146,880 (87.5\%) & 74.9\% & 38.0\%\\
Phases 0+2 & 983,040 (75.0\%) & 61.9\% & 29.5\%\\
Phases 0+2+4 & 819,200 (62.5\%) & 37.2\% & 20.7\%\\
\bottomrule
\end{tabular}
\end{table}

Exponential moving-average (EMA) summaries are also non-identical under the fixed attacks, although their normalized $L_1$ changes are smaller. The GRU remains near saturation in the intersection experiment, so adding observers does not give a universal robustness guarantee. Randomized schedules expose about 97.5\% of the saved canonical-null DVS attacks, while the schedule-blind visibility statistic remains zero; Appendix~\ref{app:controlled} gives the corresponding detector experiment.

We also verified zero fallback in the signed-displacement matcher, exact integer equality of the protected tensor, bitwise equality of protected-consumer logits, fixed evaluation populations, and clean-only detector calibration. These checks ensure that a change in the protected representation or evaluation population is not counted as stealth.

\FloatBarrier
\section{Constrained projected gradient descent (C-PGD)}
\label{app:cpgd-spec}
C-PGD searches the same unrestricted within-window exact-null feasible set as Null. Let $X_0$ be the clean grouped count tensor, with one row for each protected $(\text{window},p,y,x)$ group and one column for each fine temporal bin. The primary comparison uses one projected-gradient step per progressive budget stage ($s=1$); $s=2$ and $s=4$ are optimization-strength ablations.

\begin{algorithm}[!htbp]
\caption{C-PGD inside the exact-null feasible set}
\label{alg:cpgd}
\begin{algorithmic}[1]
\Require Clean grouped counts $X_0$, label $y$, budget stages $\mathcal B$, steps per stage $s$, nominal step $\eta=1$
\State $Z\gets X_0$
\For{$b\in\mathcal B$}
  \For{$j=1,\ldots,s$}
    \State $g\gets \nabla_Z\mathcal L(Z,y)$
    \For{each protected row $r$}
      \State $g_r\gets g_r-\operatorname{mean}(g_r)$; normalize nonzero $g_r$ by $\|g_r\|_\infty$
    \EndFor
    \State $\eta_{\mathrm{eff}}\gets\eta$; increase it as needed, for at most 20 expansions
    \State $Z\gets\Pi_{\mathrm{row}}(Z+\eta_{\mathrm{eff}}g)$, projecting each row onto $\{z\ge0:\sum_j z_j=\sum_j X_{0j}\}$
    \If{$\tfrac12\|Z-X_0\|_1>b$}
      \State Contract $Z$ toward $X_0$ to the budget boundary
    \EndIf
  \EndFor
\EndFor
\State Integerize row deficits and surpluses by deterministic largest-remainder transport
\State Materialize legal within-window source-to-destination timestamp shifts
\State \Return integer event stream $X'$ with $A(X')=A(X_0)$
\end{algorithmic}
\end{algorithm}

The row-wise mean subtraction removes the component that would change a protected row sum. Simplex projection preserves each row's total mass, and the final transport step converts the relaxed tensor to integer event counts while keeping the requested moved-event budget. Ties use stable index order. The implementation checks nonnegativity, total event count, moved-event budget, and exact coarse equality after integerization.

At $s=1$, Null and C-PGD are matched by victim-gradient evaluations: three through 10\% on DVS Gesture and four on DailyDVS-200. Wall-clock time is not matched because the optimizers perform different non-gradient operations. The comparison tests whether a second optimizer can find damaging retimings in the same exact blind space.

\end{document}